\pdfoutput=1
\documentclass{article}
\usepackage{preprint,times}

\usepackage{amsmath,amsfonts,bm}
\newcommand{\R}{\mathbb{R}}
\def\eqref#1{equation~\ref{#1}}
\usepackage{hyperref}
\usepackage{amssymb}
\usepackage{amsthm}
\theoremstyle{plain}
\newtheorem{theorem}{Theorem}[section]
\newtheorem{proposition}[theorem]{Proposition}
\theoremstyle{remark}
\newtheorem{remark}[theorem]{Remark}
\usepackage{dsfont} % blackboard-bold digits for the indicator \mathds{1}
\usepackage{algorithm}
\usepackage{algorithmic}
\usepackage{wrapfig}
\usepackage{booktabs}
\usepackage{graphicx}
\usepackage{array}
\usepackage{colortbl}
\usepackage{xcolor}
\definecolor{sparkbest}{rgb}{0.851,0.898,0.808}
\definecolor{rowshade}{gray}{0.94}
\definecolor{designsym}{rgb}{0.964,0.897,0.767}
\definecolor{thmshade}{gray}{0.95}
\usepackage[framemethod=none]{mdframed}
\mdfdefinestyle{thmbox}{backgroundcolor=thmshade,linewidth=0pt,innertopmargin=5pt,innerbottommargin=5pt,innerleftmargin=6pt,innerrightmargin=6pt,skipabove=8pt,skipbelow=8pt}
\surroundwithmdframed[style=thmbox]{theorem}
\surroundwithmdframed[style=thmbox]{proposition}
\surroundwithmdframed[style=thmbox]{remark}

\newcommand{\se}{\mathfrak{se}(3)}
\newcommand{\xih}{\hat{\bm{\xi}}}
\newcommand{\prd}{\mathrm{pred}}
\newcommand{\grt}{\mathrm{gt}}
\newcommand{\Jp}{J_{\prd}}
\newcommand{\Jg}{J_{\grt}}

\title{ArticulateArena: \\ A Metric for Articulated Kinematics}

\author{Yumeng He$^{1}$, Yongfei She$^{1,2}$, Huanyu Chen$^{2}$, Chun Yuan$^{3}$, Peihao Li$^{4,5}$, \\
\bfseries Joseph Masterjohn$^{6}$, Yin Yang$^{3}$, Ying Jiang$^{1}$, Chenfanfu Jiang$^{1}$ \\[4pt]
$^{1}$UCLA \quad
$^{2}$USC \quad
$^{3}$University of Utah \quad
$^{4}$Envora \quad
$^{5}$UCB \quad
$^{6}$Toyota Research Institute
}

\hypersetup{
  pdftitle={ArticulateArena: A Metric for Articulated Kinematics},
  pdfauthor={Yumeng He, Yongfei She, Huanyu Chen, Chun Yuan, Peihao Li, Joseph Masterjohn, Yin Yang, Ying Jiang, Chenfanfu Jiang},
  pdfsubject={Articulated object reconstruction, kinematic metric, dataset},
  pdfkeywords={articulated objects, kinematics, metric, screw theory, dataset},
}

\begin{document}

\maketitle
\vspace{-10pt}
\begin{abstract}
\vspace{-5pt}
Modern methods reconstruct or generate simulation-ready articulated objects, predicting not only their geometry but also how their parts are connected and allowed to move. Evaluating the geometry is straightforward, but evaluating the predicted articulation is not, because articulation specifies a motion rather than a shape, and there is no agreed distance between two motions. More specifically, existing protocols score joint type, axis direction, origin, and motion limits separately, although these parameters jointly describe a single physical motion, and the same motion can be written as different parameter values. As a result, a joint can score maximally wrong against an equivalent encoding of itself, and several component errors are ill-conditioned or undefined exactly where predictions become accurate. We propose \textbf{ArticulateArena}, a representation-invariant counterpart of Chamfer distance for articulation that compares the motions one-DOF joints induce rather than the parameters that encode them. It represents each joint by the unordered pair of its Lie-algebra endpoint twists, and we prove that the resulting quotient distance is a metric. It unifies fixed, revolute, prismatic, and helical joints, brings continuous joints into the same score through a compactification, and reads as the RMS motion of the moving part in meters when weighted by its mass distribution. A motion-aware tree edit distance lifts the metric to full kinematic trees, pricing structural errors such as spurious or missing joints in the same motion units as joint errors, and for a fixed inner product it remains a metric on trees up to relabeling. Alongside the metric we release \textbf{ArticulateArena-20K}, a new suite of 19{,}977 articulated objects with verified kinematics, and we re-evaluate published reconstruction methods on it under the new metric. Project page: \url{https://heyumeng.com/ArticulateArena-web/}
\end{abstract}

\vspace{-10pt}

\section{Introduction}
\label{sec:intro}
\vspace{-5pt}

Robots are trained and evaluated in simulation, and simulation benefits from running fast, ideally at many times real time. One way to speed it up is to simplify the world into rigid links with compact collision geometry, connected by joints that the simulator executes as constraints rather than by resolving collisions between objects. The constraint also stands in for the physical mechanism, so the hardware that realizes a joint never has to be modeled or held together by contact forces. Articulated objects are the assets that make this simplification possible, and a rapidly growing body of work therefore reconstructs them automatically, recovering part geometry and joint parameters, up to a full URDF with its kinematic tree of links, from several observed articulation states \citep{jiang2022ditto,liu2023paris,weng2024digitaltwin,guo2025articulatedgs,lin2025splart,wu2026reartgs} or from one or a few images \citep{heppert2023carto,chen2024urdformer,mandi2025real2code,le2025articulate,yuan2025larm,li2026art,wang2026artico,wu2026urdfanything}, while implicit shape representations with an articulation code recover joint angles without a kinematic model \citep{mu2021asdf}. Complementary systems infer articulation from a static 3D mesh, including ArtLLM, Particulate, and Articulate AnyMesh \citep{wang2026artllm,li2026particulate,qiu2025articulate}, or generate complete articulated assets, including Articraft and SPARK \citep{zhou2026articraft,he2025spark}. Most of this progress is reported through one evaluation template. For each joint it compares the predicted axis direction, the joint origin, the joint type, and, in the most complete protocols, the motion limits, four errors in four different units whose only aggregate is a thresholded success rate. The recovered geometry, by contrast, is judged by the agreed-upon Chamfer distance and F-score \citep{fan2017pointset,knapitsch2017tanks}. The recovered motion has no comparable standard. As a consequence, an asset that moves exactly like the real object can be reported as a failure, one that misses by a hair is indistinguishable from one that is wildly wrong, and the numbers of different papers cannot be compared, because each protocol scores its own subset of the components (Table~\ref{tab:protocols}).

We argue that this template, rather than the reconstruction models, is now the bottleneck of the field. Its difficulty is structural. The parameters of a joint describe one motion only together, so scoring them in isolation misreads the joint. Worse, the same motion can be written down as different parameter values, so a parameter difference need not be a motion difference at all. Reversing the axis direction and negating the limits, for instance, describes the same hinge, yet the two encodings disagree in every component. A score built on the parameters therefore grades the code instead of the motion. It also judges too coarsely, since a joint whose single component slightly misses its tolerance is discarded as completely wrong rather than scored by how far the recovered motion actually is from the truth. Table~\ref{tab:failures} and Figure~\ref{fig:perturbation} make these failures concrete on simple joints, where a score in use charges a nonzero error between a joint and itself, loses conditioning exactly where a prediction becomes accurate, or jumps across a joint-type boundary that is physically continuous (\S\ref{sec:advantages}). What one wants to measure is how far the recovered motion is from the true one, and no score in use is a distance on motions.

Our solution is to compare motions, not parameters. Classical screw theory represents the motion of a one-DOF joint by a single vector, a \emph{twist}, together with the interval the joint can travel \citep{murray1994mathematical}. We encode a joint by the unordered pair of twists at the two ends of that travel. This pair removes the representational redundancy entirely, because every way of writing down the same joint yields the same pair, up to swapping its two elements. The distance between two joints, $E$, is then simply the distance between their endpoint pairs, which we prove is a metric (\S\ref{sec:joint}). Because it compares motions directly, it treats a nearly welded revolute joint as nearly fixed instead of categorically wrong, and it stays well-behaved where the per-component scores blow up. The construction then extends along three independent axes. The norm behind $E$ is a choice, $E_\alpha$ under a dimensionless norm that treats every joint alike, or $E_B$ under a physical one that weights errors by how far they actually move the part, in meters. A compactification $E^{\phi}$ brings continuous joints into the same score instead of leaving them unscorable. A lift $E^{\mathrm{tree}}$ to kinematic trees prices structural errors, such as a spurious or missing joint, in the same motion units as joint errors. Together these pieces give articulation what Chamfer distance gives geometry, a single standard measure of error.

In detail, we make the following contributions:

\noindent (i) We define $E$, the quotient distance between the unordered pairs of Lie-algebra endpoint twists of one-DOF joints, with its compactified extension $E^{\phi}$ for continuous joints and its lift $E^{\mathrm{tree}}$ to kinematic trees, and prove that both are metrics for a fixed inner product (\S\ref{sec:joint}).

\noindent (ii) We show concretely which metric axiom each per-component score in current use violates and where each degenerates, then re-evaluate published methods under the quotient construction, which induces different orderings on their outputs (\S\ref{sec:baselines}).

\noindent (iii) We introduce \textbf{ArticulateArena-20K}, a standard suite spanning 15 supercategories, 628 categories, and 19{,}977 articulated objects with ground-truth kinematics. The library covers all five one-DOF joint types, including the continuous joints that prior protocols cannot score (\S\ref{sec:dataset}).

\vspace{-10pt}
\section{Related work}
\label{sec:related}
\vspace{-10pt}
\textbf{Articulated objects.} Datasets and part-level perception supply the representations on which the reconstruction and generation systems cited in \S\ref{sec:intro} build. PartNet and PartNet-Mobility provide part annotations and simulator-ready joints \citep{mo2019partnet,xiang2020sapien}, and perception methods predict motion parts, openable parts, actionable segments, or category-level part poses and articulation states \citep{wang2019shape2motion,jiang2022opd,geng2023gapartnet,liu2023partslip,li2020ancsh,weng2021captra}, and factor-graph tracking recovers joint parameters from observed motion without category priors \citep{heppert2022category}. Whatever their input, the reconstruction systems score recovered kinematics with the per-component template that \S\ref{sec:prev-eval} examines. We do not propose another reconstructor, but a representation-invariant way to compare the articulated models these systems produce.

\textbf{Rigid-motion metrics and structured matching.} Coordinate-independent norms and metrics for rigid displacements and rotations \citep{kazerounian1992object,park1995distance,huynh2009metrics} compare poses or individual displacements, not joints with bounded or unbounded motion ranges, and the componentwise scores above make no use of them. We instead compare the full admissible motion encoded by each joint. At the object level, tree-edit distances \citep{zhang1989simple,pawlik2016tree} price topology errors and assignment handles part relabeling \citep{kuhn1955hungarian}; our lift connects them to the joints through a single motion-priced edit cost, rather than treating joint parameters as unrelated fields.

\vspace{-5pt}
\section{Preliminaries}
\label{sec:problem}
\vspace{-5pt}
\begin{figure}[t]
\centering
\includegraphics[width=\textwidth]{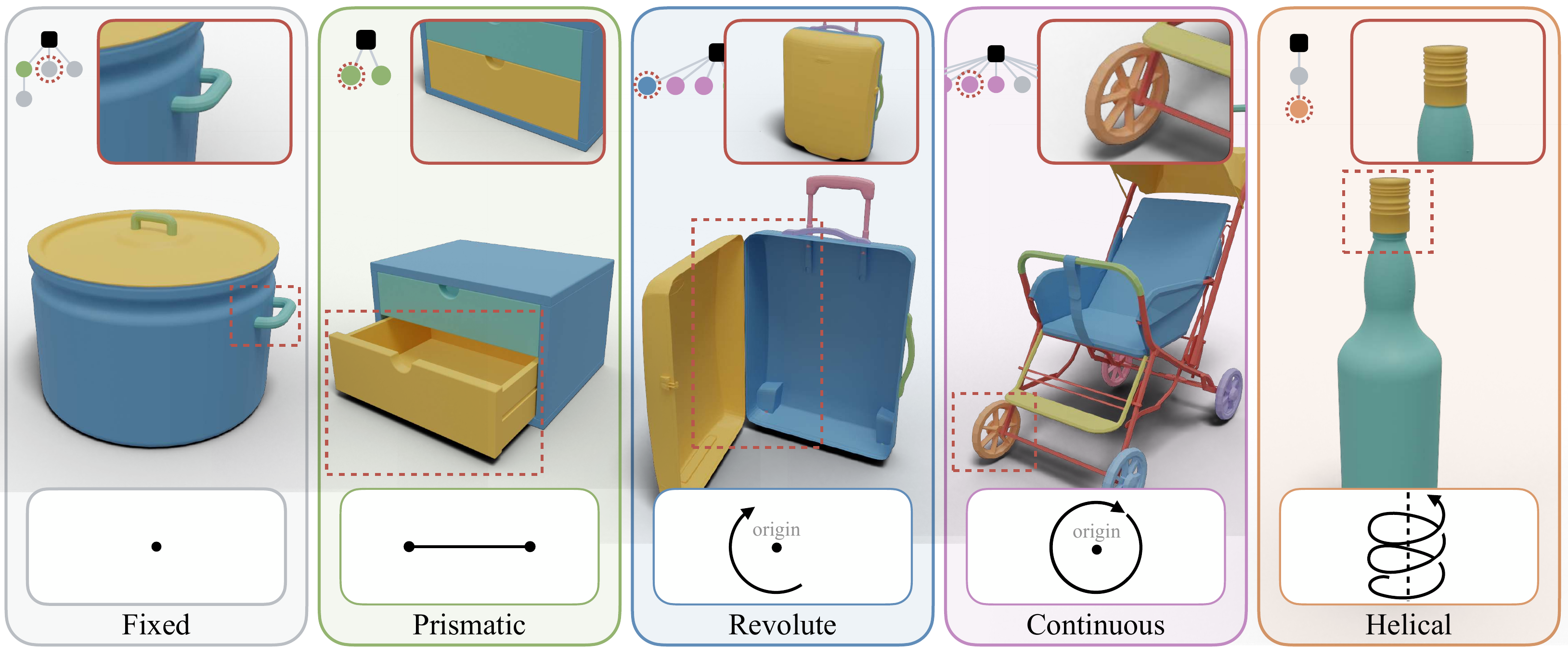}
\vspace{-20pt}
\caption{Categories of joints in articulated objects: fixed, prismatic, revolute, continuous, and helical. Each panel shows an object with one part displaced along its joint (dashed box), an inset of the same part at rest, the kinematic tree with that joint highlighted, and an icon of the motion the joint allows. These five one-DOF categories are exactly the joints our metric covers.}
\label{fig:type}
\vspace{-10pt}
\end{figure}

\subsection{Articulated objects}
\label{sec:articulated}
\vspace{-5pt}

An articulated object is a set of rigid links connected by joints into a kinematic tree, usually encoded as a URDF. Each link carries a mesh or a set of collision shapes, and each non-root link is attached to its parent by exactly one joint. A one-DOF joint (Appendix~\ref{app:onedof}) is described by four pieces of information: a \emph{type} $t$ (Figure~\ref{fig:type}); an \emph{axis} $\bm{a}$, the unit direction of rotation or translation; an \emph{origin} $\bm{o}$, a point the axis passes through; and a \emph{limit} $l = [l^-, l^+]$, the range of the joint coordinate. A revolute joint with $l^+ - l^- = \infty$ is \emph{continuous}. The axis and origin fix the \emph{line} along which the part moves, the limits fix the \emph{extent} of the motion, and the type says which kind of motion it is. Together they form the tuple
\begin{equation}
\label{eq:tuple}
J \;=\; (t,\; \bm{a},\; \bm{o},\; l),
\end{equation}
which a reconstruction method predicts for every joint and an evaluation protocol must score against the ground truth, $\Jp$ against $\Jg$.

\begin{wraptable}[16]{l}{0.41\textwidth}
\vspace{-5pt}
\centering
\small
\setlength{\abovecaptionskip}{0pt}
\setlength{\belowcaptionskip}{4pt}
\caption{Notation used throughout the paper.}
\label{tab:notation}
\begin{tabular}{ll}
\toprule
Symbol & Meaning \\
\midrule
$J$ & a one-DOF joint \\
$t$ & joint type \\
$\bm{a}$ & joint axis (unit) \\
$\bm{o}$ & joint origin \\
$l = [l^-, l^+]$ & limit interval \\
$q$ & joint coordinate (rad or m) \\
$\bm{\xi} = (\bm{\omega}, \bm{v})$ & screw twist \\
$\bm{\omega}$ & rotational part of $\bm{\xi}$ \\
$\bm{v}$ & translational part of $\bm{\xi}$ \\
$\xih$ & matrix form of $\bm{\xi}$ \\
$\se$ & Lie algebra of $SE(3)$ \\
\bottomrule
\end{tabular}
\end{wraptable}

\subsection{Previous evaluation protocols}
\label{sec:prev-eval}

Existing reconstruction and generation methods evaluate predicted kinematics through one componentwise template, comparing the tuple $J$ of \eqref{eq:tuple} entry by entry. The most common form, of which the other protocols in use score subsets (Table~\ref{tab:protocols}, Appendix~\ref{app:protocols}), scores the following components \citep{le2025articulate}; URDFormer, one of the image-to-URDF pipelines above, reports only object- and part-level category, parent, and discretized position accuracy instead. They carry different units and cannot be summed, so the protocol thresholds each one and reports a \emph{success rate}, the fraction of joints whose every component falls within its own tolerance. Appendix~\ref{app:succ} states this aggregation precisely and discusses what it discards, and \S\ref{sec:advantages} works through the failures of each component and shows how our metric resolves them.

\noindent (i) \emph{Joint type error}, a binary measure of whether the predicted joint type matches the ground truth, $e_{\mathrm{type}} = \mathbf{1}[t_{\prd} \neq t_{\grt}]$.

\noindent (ii) \emph{Joint axis error}, the angle between the predicted and ground-truth joint axes up to sign, $e_{\mathrm{axis}} = \arccos\big(\lvert \bm{a}_{\prd} \cdot \bm{a}_{\grt}\rvert / (\lVert \bm{a}_{\prd}\rVert_2\,\lVert \bm{a}_{\grt}\rVert_2)\big) \in [0,\pi/2]$, where $\bm{a}$ is the axis of rotation for revolute joints or of translation for prismatic joints.

\noindent (iii) \emph{Joint origin error}, the shortest distance between the joint axes for revolute joints and the Euclidean distance between origins for prismatic joints, with $\bm{p} = \bm{o}_{\prd} - \bm{o}_{\grt}$ given by $e_{\mathrm{orig}}^{\mathrm{rev}} = \lvert\, \bm{p} \cdot (\bm{a}_{\prd} \times \bm{a}_{\grt})\,\rvert / \lVert \bm{a}_{\prd} \times \bm{a}_{\grt}\rVert$ and $e_{\mathrm{orig}}^{\mathrm{pris}} = \lVert \bm{p} \rVert_2$.

\noindent (iv) \emph{Joint limit error}, comprising two sub-components. With the motion vectors $\bm{m}_j = \bm{a}_j\,(l_j^+ - l_j^-)$ for $j \in \{\prd, \grt\}$, where $l_j^+$ and $l_j^-$ are the upper and lower joint limits, the motion-range and motion-direction differences are $e_{\mathrm{lim}}^{\mathrm{range}} = \lVert \bm{m}_{\prd} - \bm{m}_{\grt}\rVert_2$ and $e_{\mathrm{lim}}^{\mathrm{dir}} = 1 - \bm{m}_{\prd} \cdot \bm{m}_{\grt} / (\lVert \bm{m}_{\prd}\rVert_2\,\lVert \bm{m}_{\grt}\rVert_2)$, the latter ranging from $0$ (identical direction) through $1$ (perpendicular) to $2$ (opposite).

\vspace{-5pt}
\section{The quotient metric}
\label{sec:joint}
\vspace{-5pt}
\begin{figure}[t]
\centering
\includegraphics[width=\textwidth]{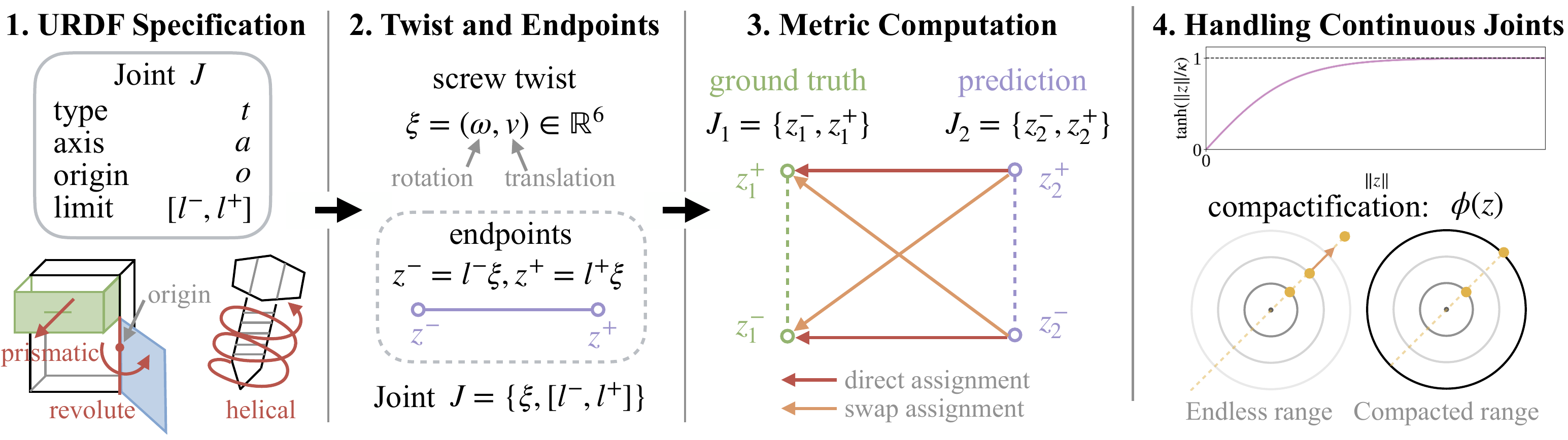}
\vspace{-20pt}
\caption{From a URDF joint to our metric. (1) A joint is specified by type, axis, origin, and limits. (2) The twist $\bm{\xi} = (\bm{\omega}, \bm{v})$ absorbs type, axis, and origin, and the limits give the endpoint pair $\bm{z}^\pm = l^\pm\bm{\xi}$. (3) $E$ compares two endpoint pairs under the better of the direct and the swapped assignment. (4) Radial compactification maps unbounded ranges into the unit ball, so continuous joints become boundary points.}
\label{fig:pipeline}
\vspace{-5pt}
\end{figure}
We define a joint by a single \emph{twist} $\bm{\xi} = (\bm{\omega}, \bm{v}) \in \R^6$ together with its limit interval. The twist will absorb the type, the axis, and the origin at once. Here $\bm{\omega}$ is the rotational part and $\bm{v}$ the translational part of an infinitesimal rigid motion. The hat map writes $\bm{\xi}$ as the matrix $\xih = \big(\begin{smallmatrix} \hat{\bm{\omega}} & \bm{v} \\ 0 & 0 \end{smallmatrix}\big) \in \se$, with $\hat{\bm{\omega}}$ the skew-symmetric matrix of $\bm{\omega}$, and the matrix exponential $\exp(q\,\xih) \in SE(3)$ is the rigid motion reached by flowing along $\bm{\xi}$ for $q$ units \citep{murray1994mathematical}. A drawer slides along a prismatic twist, a door swings about a revolute twist, and a helical joint couples the two motions along one axis. Figure~\ref{fig:pipeline} summarizes the construction of this section. We normalize $\bm{\xi}$ to a \emph{unit} screw, with $\lVert\bm{\omega}\rVert = 1$ when $\bm{\omega} \neq 0$ and $\lVert\bm{v}\rVert = 1$ otherwise, and call $q$ the \emph{joint coordinate}, the signed amount of motion away from the rest state. This normalization makes $q$ a rotation angle in radians when $\bm{\omega} \neq 0$ (revolute, helical) and a translation length in meters when $\bm{\omega} = 0$ (prismatic), and the limits $l^\pm$ carry the same unit. The twist packs rotation and translation into one object, and the tuple parameters of \S\ref{sec:articulated} embed as special cases:
\begin{equation}
\label{eq:embed}
\bm{\xi} \;=\; (\bm{\omega},\, \bm{v}), \qquad
\bm{\xi}_{\mathrm{rev}} \;=\; (\bm{a},\; \bm{o} \times \bm{a}), \qquad
\bm{\xi}_{\mathrm{pris}} \;=\; (0,\; \bm{a}), \qquad
\bm{\xi}_{\mathrm{hel}} \;=\; (\bm{a},\; \bm{o} \times \bm{a} + h\,\bm{a}),
\end{equation}
with $\bm{a}$ the axis, $\bm{o}$ the origin, and $h$ the pitch. The moment $\bm{o} \times \bm{a}$ encodes the axis position without naming a point on it (Appendix~\ref{app:twists} derives these embeddings).

We define a one-DOF joint by its motion, the curve of child-link placements $g(q) = \exp(q\,\xih)\,g_0$ in $SE(3)$, $q \in [l^-, l^+]$, with $\bm{\xi} \in \se$ a unit screw twist and the anchor $g_0$ the placement of the link at $q = 0$, the state in which the URDF stores its meshes. Prediction and ground truth describe the same object in the same stored state, so the anchor is shared data whose accuracy the shape metrics judge, and the kinematic content of the joint is the pair
\begin{equation}
\label{eq:jointdef}
J \;=\; \big(\bm{\xi},\; [l^-, l^+]\big) \;\in\; \se \times \overline{\R}^2,
\end{equation}
replacing the componentwise tuple $(t, \bm{a}, \bm{o}, l)$ of \S\ref{sec:problem}. The interval $[l^-, l^+]$ is extended-real, so that continuous joints have $l^\pm = \pm\infty$. Equivalently we represent the joint by its unordered pair of \emph{endpoint twists}
\begin{equation}
\label{eq:endpoints}
\{\bm{z}^-, \bm{z}^+\} = \{l^-\,\bm{\xi},\; l^+\,\bm{\xi}\} \subset \se,
\end{equation}
which keeps limits as first-class data and has no log/winding ambiguity. Re-zeroing the coordinate is not a free re-encoding, because $q \mapsto q - \delta$ moves the anchor to $\exp(\delta\,\xih)\,g_0$, which the stored state pins down. One representational symmetry remains (Proposition~\ref{prop:same_motion}). Reversing the screw orientation swaps the endpoints, and the unordered pair absorbs it. The revolute exponential is $2\pi$-periodic, so limits shifted by a full turn generate the same placements, but we keep the joint coordinate as part of the specification and do not quotient this shift, which keeps $E$ continuous in the limits (Remark~\ref{rem:anchor}). Two anchored encodings then have the same endpoint pair iff they agree up to orientation reversal, and the proposed metric is the quotient distance of flat $\se\times\se$ under the swap:
\begin{equation}
\label{eq:E}
\boxed{\;E(J_1,J_2) \;=\; \min\Big\{ \lVert \bm{z}^-_{1}{-}\bm{z}^-_{2}\rVert^2 + \lVert \bm{z}^+_{1}{-}\bm{z}^+_{2}\rVert^2,\;\; \lVert \bm{z}^-_{1}{-}\bm{z}^+_{2}\rVert^2 + \lVert \bm{z}^+_{1}{-}\bm{z}^-_{2}\rVert^2 \Big\}^{1/2}\;}
\end{equation}
\begin{wraptable}[11]{l}{0.50\textwidth}
\centering
\small
\setlength{\abovecaptionskip}{0pt}
\setlength{\belowcaptionskip}{4pt}
\caption{The design space. An inner product (rows) combined with a limit treatment or the tree lift (columns). Shaded entries mark the reported variant.}
\label{tab:design}
\setlength{\tabcolsep}{4pt}
\begin{tabular}{@{}lccc@{}}
\toprule
 & Finite & Unbounded & Trees \\
\midrule
Split norm $\lVert\cdot\rVert_\alpha$ & $E_\alpha$ & \colorbox{designsym}{$E_\alpha^{\phi}$} & \colorbox{designsym}{$E_\alpha^{\phi,\mathrm{tree}}$} \\[3pt]
Kinetic norm $\lVert\cdot\rVert_B$ & \colorbox{designsym}{$E_B$} & \colorbox{designsym}{$E_B^{\phi}$} & $E_B^{\phi,\mathrm{tree}}$ \\
\bottomrule
\end{tabular}
\end{wraptable}
a true metric, being a quotient by a finite isometry group (Theorem~\ref{thm:endpoint_metric}). The norm $\lVert\cdot\rVert$ comes from an inner product on $\se$ chosen in \S\ref{sec:design}, and a subscript names that choice (Table~\ref{tab:design}): $E_\alpha$ under the split norm and $E_B$ under the kinetic-energy norm. Superscripts mark extensions that apply under either subscript, $E^{\phi}$ for the compactified metric of \S\ref{sec:infinite} and $E^{\mathrm{tree}}$ for the tree lift of \S\ref{sec:lift}. It is zero iff the two encodings agree up to orientation reversal (Proposition~\ref{prop:same_motion}), smooth except on a measure-zero set of ties, invariant under screw-orientation reversal and the choice of prismatic origin (Propositions~\ref{prop:axis_flip} and~\ref{prop:prismatic_gauge}), and it identifies the joint from its endpoint pair up to orientation (Proposition~\ref{prop:identifiability}). Theorem~\ref{thm:endpoint_path} further bounds $E$ on both sides by the $L^2$ discrepancy of the full linearized motion segment, so comparing endpoints already controls the whole path. Appendix~\ref{app:theory} states and proves these properties. $E$ is symmetric, so we write its arguments as $J_1$ and $J_2$ rather than $\Jp$ and $\Jg$. The hard min makes $E$ itself nondifferentiable at ties. When a differentiable score is wanted, mainly as a training loss, Appendix~\ref{app:agg} gives smooth swap-invariant surrogates that replace the min. Appendix~\ref{app:algorithm} states the full computation, at the joint and the tree level, in pseudocode.

\subsection{Inner product on \texorpdfstring{$\se$}{se(3)}}
\label{sec:design}
\vspace{-5pt}

The norm in \eqref{eq:E} acts on endpoint-twist differences $\Delta\bm{z} = (\Delta\bm{\omega}, \Delta\bm{v})$, and Theorem~\ref{thm:endpoint_metric} asks only that it come from an inner product on $\se$. The most intuitive choice is the \emph{split norm}, which weighs rotation against translation through an inverse-length $\alpha$:
\begin{equation}
\label{eq:splitnorm}
\lVert\Delta\bm{z}\rVert_\alpha^2 \;=\; \lVert\Delta\bm{\omega}\rVert^2 + \alpha^2\lVert\Delta\bm{v}\rVert^2.
\end{equation}
The resulting $E_\alpha$ is dimensionless and satisfies the metric axioms (Theorem~\ref{thm:endpoint_metric}), and it compares every joint of every object by one fixed rule. That uniformity is also its flaw, since a $5^\circ$ axis error scores the same on a $1$\,m door and on a $2$\,cm knob, although the door moves fifty times more material.

The \emph{kinetic-energy form} lets the moving link set the scale instead. It is induced by the child link's body $B$ through its mass $m$, center of mass $\bm{c}$, and inertia tensor $I$:
\begin{equation}
\label{eq:kineticnorm}
\lVert\Delta\bm{z}\rVert_B^2 \;=\; \frac{1}{m}\,\Delta\bm{\omega}^{\!\top}\! I\,\Delta\bm{\omega} + \lVert \Delta\bm{v} + \Delta\bm{\omega}\times \bm{c}\rVert^2.
\end{equation}
The resulting $E_B$ is measured in meters and equals the RMS linearized material motion of the child link (Theorem~\ref{thm:material_motion}), so the door now scores fifty times worse than the knob. The two forms are close relatives. For a centered link with isotropic inertia $I = m r^2 I_3$, where $r$ is the radius of gyration, \eqref{eq:kineticnorm} is exactly $r^2$ times \eqref{eq:splitnorm} at $\alpha = 1/r$. A global $\alpha$ therefore assumes that every link has the same size $1/\alpha$, and the kinetic-energy form replaces that assumption with the moments of the link that actually moves. Its price is that it needs those moments. The joint must come with a link geometry, the score inherits the object's scale, and the norm varies from joint to joint, which keeps the per-link variant of the tree lift outside the metric proof (\S\ref{sec:lift}). $E_B$ itself remains a true metric, because the body $B$ is fixed by the ground truth before any prediction is seen, so each joint carries one fixed inner-product norm and Theorem~\ref{thm:endpoint_metric} applies verbatim. We keep both norms, $E_\alpha$ for geometry-free uniform comparison and $E_B$ as the physically calibrated default of our evaluation.

\subsection{Unbounded joints via radial compactification}
\label{sec:infinite}
\vspace{-5pt}
For unbounded joints we map $\se$ into its unit ball by radial compactification and measure the quotient distance there:
\begin{equation}
\label{eq:Ephi}
\phi(\bm{z}) \;=\; \tanh\!\big(\lVert \bm{z}\rVert/\kappa\big)\,\frac{\bm{z}}{\lVert \bm{z}\rVert},
\quad
E^{\phi}(J_1, J_2) \;=\; E\big(\{\phi(\bm{z}_1^-),\, \phi(\bm{z}_1^+)\},\; \{\phi(\bm{z}_2^-),\, \phi(\bm{z}_2^+)\}\big)
\end{equation}
A continuous joint becomes the antipodal boundary pair $\{\pm\bm{\xi}\}$, i.e.\ the axis direction up to sign. The completion of the space of limited joints contains the continuous joints as boundary points, and $E^{\phi}$ is continuous in this limit (Proposition~\ref{prop:continuous_limit}). The construction applies under either inner product of \S\ref{sec:design}, with $\lVert\cdot\rVert$ in $\phi$ the chosen norm, giving $E_\alpha^{\phi}$ and $E_B^{\phi}$. Both are dimensionless, including $E_B^{\phi}$. Dividing by $\lVert\bm{z}\rVert$ inside $\phi$ gives each compactified endpoint the inverse units of the norm, and the outer norm in $E$ cancels them again, so $\lVert\phi(\bm{z})\rVert = \tanh(\lVert\bm{z}\rVert/\kappa)$ and every distance between compactified pairs is a pure number (Remark~\ref{rem:scale}). $E_B^{\phi}$ thus trades the meter reading of $E_B$ for boundedness while keeping the body-aware weighting of error directions, and $\kappa$ carries the unit of the norm. Although motivated by unbounded joints, $\phi$ is defined on all of $\se$, so every joint type can be compactified. Evaluating all pairs under $E^{\phi}$ therefore unifies finite and continuous joints in a single dimensionless score, which \S\ref{sec:baselines} uses as the common score over all joint types.

\begin{figure}[t]
\centering
\includegraphics[width=\textwidth]{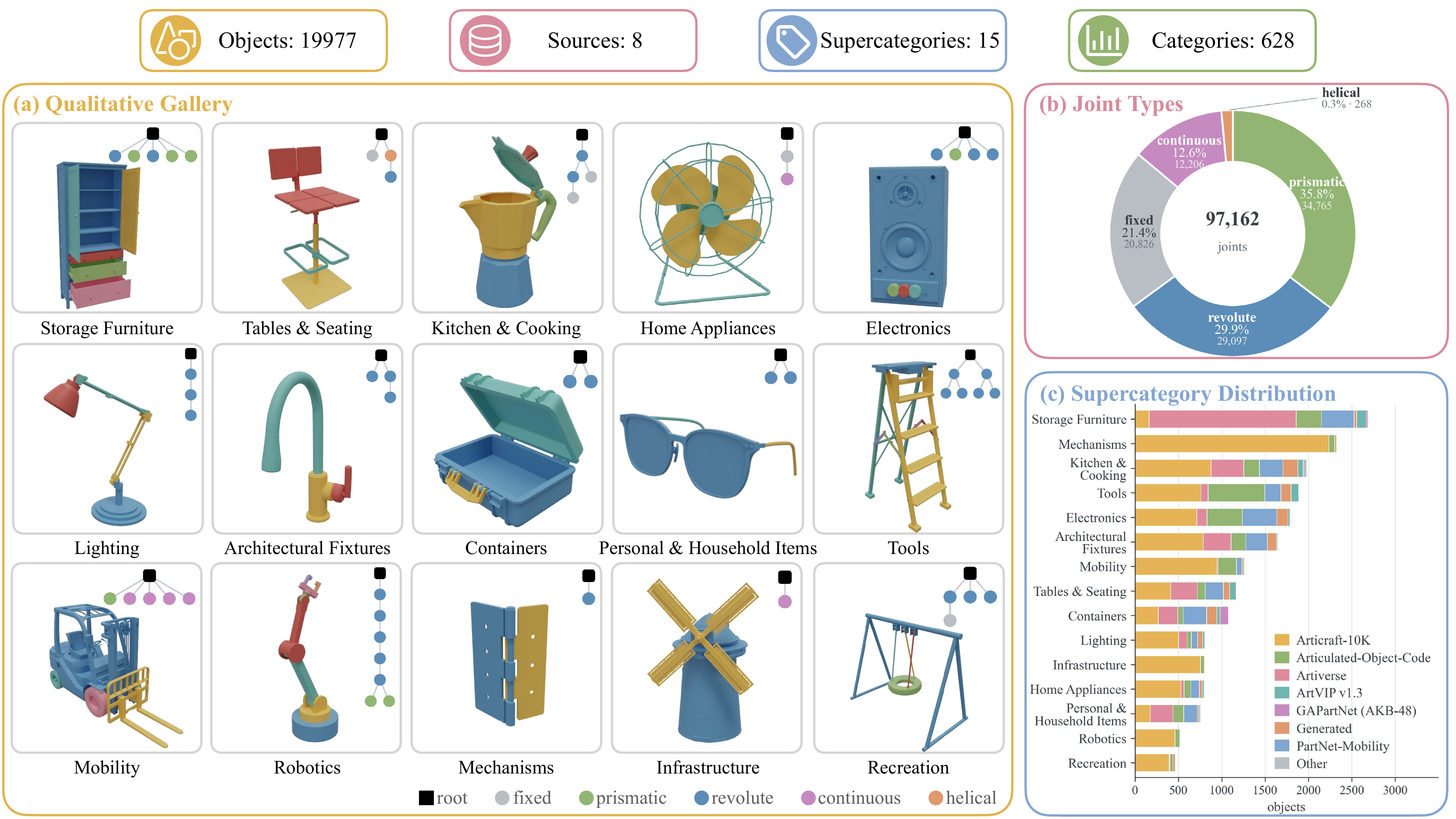}
\vspace{-20pt}
\caption{Overview of \textbf{ArticulateArena-20K}. (a) One object from each of the 15 supercategories, rendered with its kinematic tree (top right). (b) The distribution of the library's joints over the five joint types. (c) Objects and categories per supercategory, stacked by source library.}
\label{fig:dataset-overview}
\vspace{-10pt}
\end{figure}

\subsection{Lifting to skeleton and geometry}
\label{sec:lift}
\vspace{-5pt}
The metric so far covers a single joint, but an object has several joints arranged in a kinematic tree, with links as nodes and joints on the edges. The same object admits many such trees, since parts can be relabeled, the base link re-chosen, and a handle either made its own link welded to the drawer or left part of the drawer mesh. The object-level distance must therefore live on trees modulo these choices. The fixed joint $J_0 = \{0, 0\}$ is a canonical origin, and $E(J, J_0)$ measures how much motion a joint carries. We lift $E$ to trees, under either inner product of \S\ref{sec:design}, as an edit distance whose every cost comes from $E$ itself. An edit \emph{substitutes} the joint on an edge at cost $E(J, J')$, \emph{contracts} an edge by merging its two links at cost $E(J, J_0)$, or \emph{expands} a link into two joined by a new joint $J$ at the same cost. The tree distance is the cheapest edit sequence from one tree to the other,
\begin{equation}
\label{eq:dtree}
E^{\mathrm{tree}}(T_1, T_2) \;=\; \inf_{\pi\,:\,T_1 \rightsquigarrow T_2}\; \sum_{o \in \pi} c(o),
\end{equation}
over finite edit sequences $\pi$ on unrooted trees with unlabeled nodes. No hand-set topology penalty appears, because a spurious or missing joint costs exactly the motion it carries. A nearly welded extra handle is charged $\epsilon\lVert\bm{\xi}\rVert$ rather than a constant, so topology and joint discrepancies are measured in the same motion space.

On raw URDF trees $E^{\mathrm{tree}}$ is only a pseudometric, since a tree with an exact fixed edge and the tree with that edge contracted are at distance zero. Call two trees \emph{representation equivalent} when they differ only by node relabeling, base-link choice, and insertion or deletion of exact fixed edges. On these classes $E^{\mathrm{tree}}$ is a metric (Theorem~\ref{thm:tree_metric}). We compute $E^{\mathrm{tree}}$ exactly, certified through its assignment relaxation \citep{kuhn1955hungarian} (Appendix~\ref{app:tree}). Geometry enters through the inner product. The configuration discrepancy $d(g_1,g_2)$ of a link between two placements, the mass-weighted $L^2$ displacement of its material points, depends on the mass measure only through its moments $(m, \bm{c}, I)$ \citep{kazerounian1992object,park1995distance} (Appendix~\ref{app:material}), which gives $E_B$ its geometry-aware reading (\S\ref{sec:design}). Instantiating the edit costs per joint with the ground-truth child link's kinetic-energy form gives the physically weighted variant of our evaluation, a metric for a fixed skeleton and correspondence (Theorem~\ref{thm:fixed_tree}), with only the topology edits under these edge-dependent norms outside the theorems (\S\ref{sec:design}).

\subsection{The ArticulateArena-20K dataset}
\label{sec:dataset}
\vspace{-5pt}
A metric is most useful when it is computed on a standard suite of objects, so that numbers are comparable across papers, and such a suite needs ground-truth kinematics that are trustworthy across all joint types. We therefore assemble \textbf{ArticulateArena-20K}, a suite of 19{,}977 articulated objects in 628 categories and 15 supercategories that merges public articulated-object libraries with objects we generated ourselves, all normalized into a common URDF convention with verified joint annotations so that every object exposes the tuple of \S\ref{sec:articulated} and, equivalently, the endpoint-twist representation of this section. It covers the five joint types of Figure~\ref{fig:type}, continuous joints included, and ranges from single-joint items to assemblies with over a hundred movable joints. Figure~\ref{fig:dataset-overview} summarizes the composition, and Appendix~\ref{app:library} details the sources, the construction pipeline, and the category taxonomy. The empirical study of \S\ref{sec:baselines} evaluates on 200 objects chosen from the library to cover every supercategory and as many categories as possible.

\section{Experiments}
\label{sec:baselines}
\vspace{-5pt}

\subsection{Where per-component scores fail and \texorpdfstring{$E$}{E} succeeds}
\label{sec:advantages}
\vspace{-5pt}

Table~\ref{tab:failures} collects the canonical cases in which the scores of \S\ref{sec:prev-eval} report a nonzero, ill-conditioned, or undefined error while $E$ returns the correct value under either inner product, taking Articulate-Anything as the representative protocol. Appendix~\ref{app:failures} works through each case.

\begin{table}[!htb]
\centering
\small
\setlength{\abovecaptionskip}{0pt}
\setlength{\belowcaptionskip}{4pt}
\caption{Failure cases of per-component scores. Rows 1 and 4 compare two encodings of the same motion, row 2 two different motions that a component score cannot separate, and rows 3 and 5 two motions that converge to each other.}
\label{tab:failures}
\setlength{\tabcolsep}{4pt}
\begin{tabular}{>{\raggedleft\arraybackslash}p{1.0em}>{\raggedright\arraybackslash}p{0.27\textwidth}>{\raggedright\arraybackslash}p{0.34\textwidth}>{\raggedright\arraybackslash}p{0.28\textwidth}}
\toprule
 & Case & Per-component score & Ours \\
\midrule
\rowcolor{rowshade}1 & Axis reversed, limits negated & $e_{\mathrm{lim}}^{\mathrm{dir}} = 2$, maximal against itself & $E = 0$, endpoints swap \\[2pt]
2 & Interval shifted, same width & $e_{\mathrm{lim}} = 0$ although the configurations differ & $E > 0$ \\[2pt]
\rowcolor{rowshade}3 & Near-parallel axes & $e_{\mathrm{orig}}^{\mathrm{rev}}$ divides by $\lVert\bm{a}_{\prd}\times\bm{a}_{\grt}\rVert \to 0$, undefined at parallel & no denominator, shrinks continuously \\[2pt]
4 & Prismatic origin moved along the part & $e_{\mathrm{orig}}^{\mathrm{pris}} > 0$ & $E = 0$, origin absent from the twist \\[2pt]
\rowcolor{rowshade}5 & Revolute range $[0,\epsilon]$ against fixed & $e_{\mathrm{type}} = 1$ for every $\epsilon > 0$ & $E = \epsilon\lVert\bm{\xi}\rVert \to 0$ \\[2pt]
6 & Continuous joint & $e_{\mathrm{lim}}$ undefined, $\bm{m}$ infinite & boundary point under $\phi$, finite ranges approach it \\[2pt]
\rowcolor{rowshade}7 & Aggregation & thresholds discard magnitude, no distance & a metric, with a differentiable surrogate \\
\bottomrule
\end{tabular}
\vspace{-10pt}
\end{table}

\subsection{Ablation}
\label{sec:ablation}

We ablate five design choices of the metric: the inner product, $E_\alpha$ against $E_B$ (Figure~\ref{fig:calibration}; Appendices~\ref{app:norm} and~\ref{app:scale}), the split-norm weight $\alpha$ (Appendix~\ref{app:alpha}), the compactification scale $\kappa$ (Appendix~\ref{app:kappa}), the two-endpoint representation against denser sampling (Appendix~\ref{app:sampling}), and the motion-aware tree lift against a constant topology penalty (Figure~\ref{fig:tree}). Appendix~\ref{app:perturbation} sweeps every single-parameter perturbation of a ground-truth joint against the prior component scores.

\textbf{Inner product.} Figure~\ref{fig:calibration} calibrates both inner products against an external ground truth, the material-motion discrepancy $d$ of \S\ref{sec:lift}. The question is not which score is larger, since their units differ, but which one tracks actual material motion. $E_B$ lies on the identity to within a few percent, and its only visible departures are the largest axis errors at large swept angles, exactly the linearization gap that Theorem~\ref{thm:material_motion} delimits. $E_\alpha$ preserves the trend on average but scatters by more than an order of magnitude at a fixed $d$, so the same physical displacement can score ten times differently depending on the link. That scatter is the empirical cost of assuming every link has size $1/\alpha$ (\S\ref{sec:design}).

\begin{wrapfigure}[17]{l}{0.66\textwidth}
\centering
\vspace{-13pt}
\includegraphics[width=\linewidth]{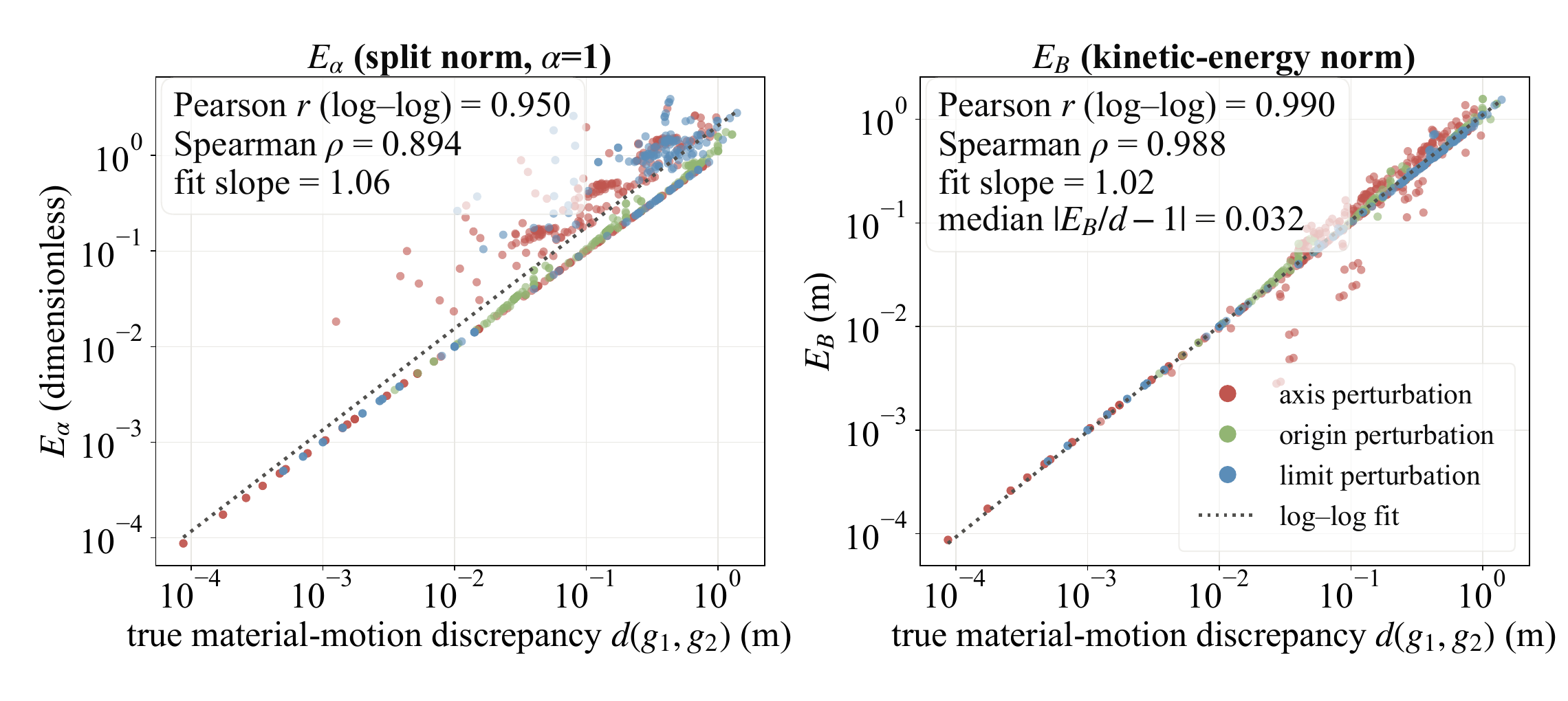}
\vspace{-25pt}
\caption{Calibration of $E_\alpha$ (left) and $E_B$ (right) against the true material-motion discrepancy $d(g_1,g_2)$ of the child link, computed on the actual $SE(3)$ displacements rather than on their linearization. Axis, origin, and limit perturbations (colors) of 690 joints from 60 objects whose child links span a wide range of sizes. The dotted line is the least-squares fit in log space.}
\label{fig:calibration}
\end{wrapfigure}

\textbf{Compactification scale.} The scale $\kappa$ in \eqref{eq:Ephi} sets where the transition between finite and continuous ranges happens, not whether it happens. Sweeping $\kappa$ around the default $\pi$ changes no qualitative property of $E^{\phi}$, only how resolution is allocated along the range axis, and $\kappa=\pi$ keeps ranges up to a full turn in the responsive part of the curve (Appendix~\ref{app:kappa}, Figure~\ref{fig:kappa}).

\textbf{Endpoint sampling.} Comparing two joints through two endpoint twists might seem to discard the motion in between, yet sampling $n$ points along each linearized segment adds nothing. The squared discrepancy is quadratic along the segment, so every $n\ge3$ coincides with the $n\to\infty$ limit, and denser sampling only reparameterizes the endpoint metric at $n/2$ times the cost (Appendix~\ref{app:sampling}, Figure~\ref{fig:sampling}).

\begin{figure}[t]
\centering
\includegraphics[width=\textwidth]{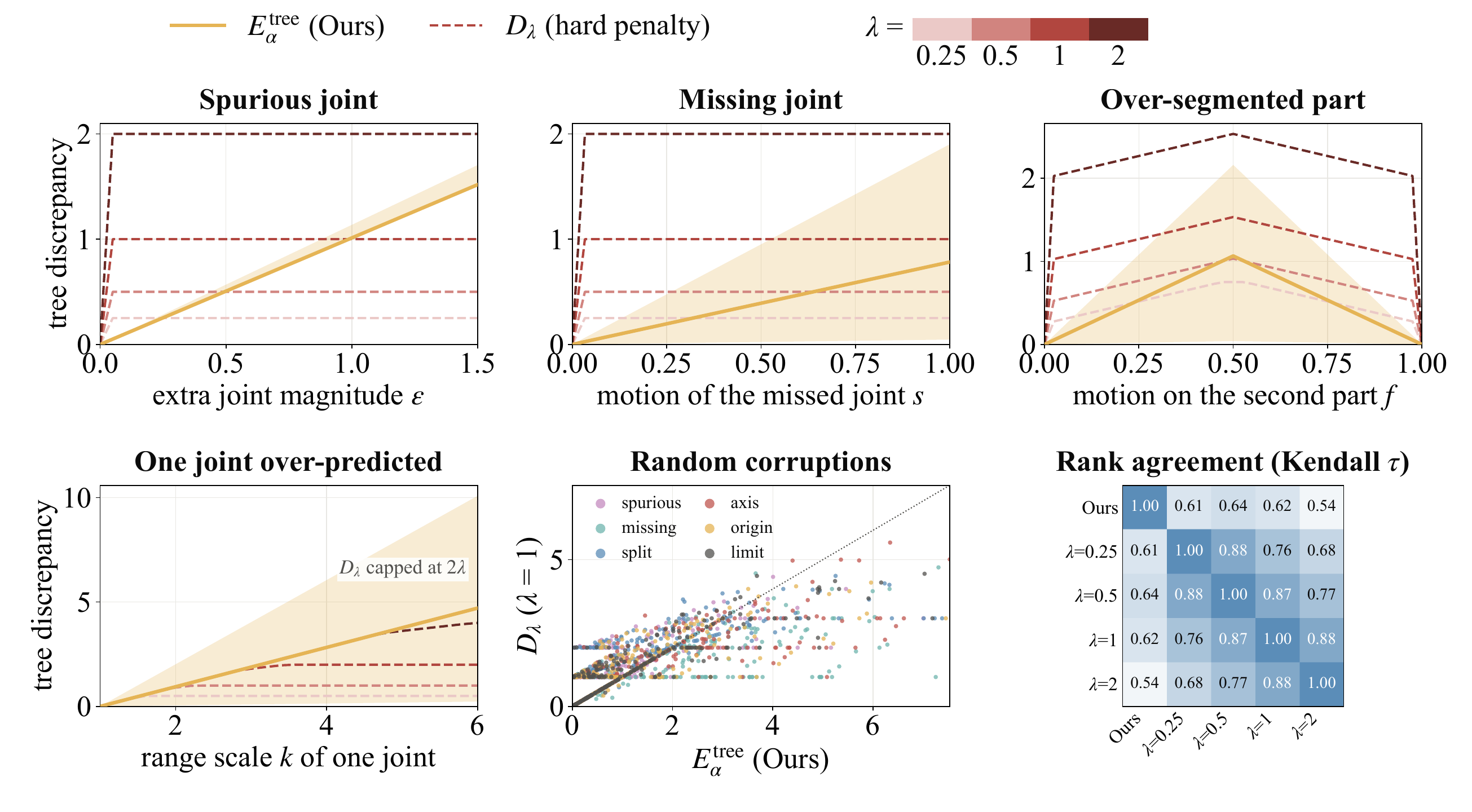}
\vspace{-30pt}
\caption{Motion-aware edit costs ($E_\alpha^{\mathrm{tree}}$ computed exactly, amber, median and interquartile range over 100 multi-joint trees from ArticulateArena-20K) against a constant topology penalty $\lambda$ per unmatched joint ($D_\lambda$, dashed, $\lambda\in\{0.25,0.5,1,2\}$). Both schemes score joints with the same metric $E_\alpha$ and differ only in what an unmatched joint costs. Top: a spurious joint of growing magnitude, a missing joint whose own motion shrinks to zero, and one part split into two along the same screw. Bottom left: the range of one joint scaled by $k$. Bottom middle: 2{,}000 randomly corrupted predictions, $E_\alpha^{\mathrm{tree}}$ against $D_1$, colored by the first edit. Bottom right: Kendall $\tau$ rank agreement among the schemes over those predictions.}
\label{fig:tree}
\end{figure}
\textbf{Tree lift.} Figure~\ref{fig:tree} compares the motion-aware edit costs of \S\ref{sec:lift} with the common alternative, a constant penalty $\lambda$ for every unmatched joint. The constant penalty is discontinuous at every structural boundary. A spurious joint that barely moves, a missing joint that barely moved, and a part split into two along its own screw all cost a full $\lambda$ the moment they appear, whereas $E_\alpha^{\mathrm{tree}}$ starts from zero and grows with the motion involved. The knob then faces a dilemma. Once a joint's own error exceeds $2\lambda$ the assignment prefers to leave it unmatched, so the score plateaus while $E_\alpha^{\mathrm{tree}}$ keeps growing, and forcing the match instead escapes the cap only by breaking the triangle inequality. A constant penalty is either capped or not a metric. On random corruptions it also piles distinct predictions onto identical scores in bands at multiples of $\lambda$, and the ranking it induces changes with the knob itself. Pricing structural edits by the motion they carry removes all of this, and it has no parameter to tune. The figure reports the exact edit distance, and its assignment relaxation returns the same value on every sweep and on nearly all random corruptions (Appendix~\ref{app:tree}).

\subsection{Comparison}
\vspace{-5pt}
\textbf{Setup.} We re-evaluate eight published methods on 200 objects from \textbf{ArticulateArena-20K} (Appendix~\ref{app:sample}). Table~\ref{tab:main} reports the prior component protocol alongside $E$ under three parameter choices of \S\ref{sec:joint}, the object-level $E_\alpha^{\phi,\mathrm{tree}}$, finite-range $E_B$, and all-joint $E_\alpha^{\phi}$ ($\alpha = 1$, $\kappa = \pi$). The tree score is the exact edit distance of \S\ref{sec:lift} (Appendix~\ref{app:tree}). Gen (\%) is computed over all objects; the remaining entries are conditional on a valid articulated output. Methods retain their native input modalities, so this compares kinematic outputs rather than common-input end-to-end performance. The success indicator, coverage, frame normalization, and geometry handling are in Appendices~\ref{app:succ} and~\ref{app:coverage}.

\begin{table}[t]
\centering
\footnotesize
\setlength{\abovecaptionskip}{0pt}
\setlength{\belowcaptionskip}{6pt}
\caption{Articulated reconstruction on 200 objects from \textbf{ArticulateArena-20K}. Arrows indicate the better direction; shading marks the best defined value; -- denotes unavailable entries. The bootstrap intervals of Appendix~\ref{app:results} show which of these gaps the evaluation set resolves. $^{\dagger}$Articulate AnyMesh predicts no motion limits, so its $E$ columns impute each joint's limits from the ground truth and are excluded from the shading.}
\label{tab:main}
\setlength{\tabcolsep}{1pt}
\begin{tabular*}{\textwidth}{@{\extracolsep{\fill}}lccccccccc@{}}
\toprule
& & \multicolumn{5}{c}{Prior per-component scores} & \multicolumn{3}{c}{Ours} \\
\cmidrule(lr){3-7}\cmidrule(l){8-10}
Method & Gen$\uparrow$ & TypeErr$\downarrow$ & OriginErr$\downarrow$ & AxisErr$\downarrow$ & LimitErr$\downarrow$ & SuccRate$\uparrow$ & $E_B\downarrow$ & $E_\alpha^{\phi}\downarrow$ & $E_\alpha^{\phi,\mathrm{tree}}\downarrow$ \\
 & (\%) & (\%) & (m) & (rad) & & (\%) & (m) & & \\
\midrule
Articraft            & 99.5 & 12.1 & 0.466 & 0.870 & 1.90\,/\,1.00 & \cellcolor{sparkbest}17.1 & 1.172 & 0.597 & 0.476 \\
Articulate AnyMesh$^{\dagger}$ & 48.0 & 16.7 & \cellcolor{sparkbest}0.264 & 0.881 & -- & 9.4 & 1.032 & 0.581 & 0.528 \\
Articulate-Anything  & 98.5 & 18.3 & 0.604 & 0.629 & 2.19\,/\,0.91 & 15.7 & 1.702 & 0.550 & \cellcolor{sparkbest}0.392 \\
ArtLLM               & 93.0 & 30.6 & 0.398 & 0.918 & 7.13\,/\,1.03 & 3.8 & 2.868 & 1.216 & 0.585 \\
Ditto                & 100.0 & 31.5 & 0.645 & 0.819 & 1.75\,/\,1.05 & 5.5 & 1.365 & 0.608 & 0.552 \\
Particulate          & 91.0 & \cellcolor{sparkbest}7.1 & 0.414 & \cellcolor{sparkbest}0.172 & \cellcolor{sparkbest}0.98\,/\,0.34 & 15.4 & \cellcolor{sparkbest}1.066 & \cellcolor{sparkbest}0.465 & 0.419 \\
SPARK                & 99.5 & 18.1 & 0.530 & 0.800 & 2.36\,/\,1.00 & 9.5 & 1.543 & 0.687 & 0.520 \\
URDFormer            & 77.0 & 37.7 & 0.735 & 1.019 & 2.07\,/\,0.98 & 0.0 & 1.964 & 0.726 & 0.566 \\
\bottomrule
\end{tabular*}
\vspace{-15pt}
\end{table}

\textbf{Results.}
On these outputs the quotient scores induce different orderings from the thresholded component protocol, especially when axis, origin, and limit errors trade off. This is expected because $E_B$ scores the encoded motion jointly, absorbs representation symmetries, and weights errors by their effect on the ground-truth body. Detailed comparisons, qualitative cases (Figure~\ref{fig:qualitative}), and common-set uncertainty are reported in Appendices~\ref{app:detailed-comparison} and~\ref{app:intersection}.

\vspace{-5pt}
\section{Conclusion}
\vspace{-10pt}
We introduced a quotient metric on one-DOF joints that compares the motions the joints generate rather than the parameters encoding them, places all five joint types, continuous included, on one scale, and reads as the RMS motion of the child link in meters under the kinetic-energy inner product. Lifted to kinematic trees, it replaces the per-component scores in current use with a representation-invariant distance that satisfies the metric axioms for a fixed inner product and stays well-conditioned where those scores degenerate. We hope that it, together with ArticulateArena-20K, becomes the default way to report articulated-reconstruction quality. Multi-DOF joints, closed kinematic chains, and sharper swept-motion bounds under the kinetic-energy form are left to future work.

\section*{AI use statement}
We used generative AI tools in three places. First, in the generation of the dataset, where part of ArticulateArena-20K was produced with generative models (Appendix~\ref{app:construction}). Second, in the organization of the dataset, where an LLM proposed a category label for every object from its name, description, and renders. Third, to refine text written by the authors, to re-check the mathematical proofs written by the authors, and to help with \LaTeX{} layout. We have reviewed all AI-assisted work. Every generated asset was loaded and checked for validity. Every proposed category label was reviewed and corrected by the authors, so the final categorization and supercategory taxonomy (Appendix~\ref{app:supercats}) reflect human judgment. All refined text was checked by the authors for accuracy, and every proof was verified by the authors independently of the AI re-check. We take responsibility for the final content of this work, including text, claims, and artifacts produced with the aid of generative AI.

\section*{Ethics statement}
The one ethical aspect of this work is dataset release. ArticulateArena-20K redistributes assets from publicly released research datasets under their respective licenses, with attribution in \S\ref{sec:dataset} and Appendix~\ref{app:library}, alongside synthetic assets we generated ourselves. The library describes household objects and contains no personal data or human-subject material.

\section*{Reproducibility statement}
The metric is fully specified by \eqref{eq:E}, \eqref{eq:splitnorm}, \eqref{eq:kineticnorm}, \eqref{eq:Ephi}, and \eqref{eq:dtree}, with the screw-twist embeddings of a URDF joint in Appendix~\ref{app:twists} and the differentiable training surrogates in Appendix~\ref{app:agg}. Every theoretical claim is stated with its assumptions and proved in Appendix~\ref{app:theory}. The per-component scores and the success indicator used for comparison are written out in \S\ref{sec:prev-eval} and Appendix~\ref{app:protocol_details}, evaluation counts and the common-set analysis in Appendix~\ref{app:results}, and the ablation setups in \S\ref{sec:ablation} and Appendix~\ref{app:ablations}. The construction, normalization, and quality control of ArticulateArena-20K are described in Appendix~\ref{app:library}. We will release the evaluation code, the ablation scripts, and the library with its URDF assets and verified joint annotations.

\bibliography{references}
\bibliographystyle{preprint}

\clearpage
\appendix
\raggedbottom

\begingroup
\small
\newcommand{\tocS}[2]{\noindent\textbf{\makebox[1.6em][l]{\ref{#1}}#2}\dotfill\textbf{\pageref{#1}}\par}
\newcommand{\tocs}[2]{\noindent\hspace{1.6em}\makebox[2.2em][l]{\ref{#1}}#2\dotfill\pageref{#1}\par}
\noindent\textbf{Appendix contents}\par\vspace{4pt}
\tocS{app:details}{Details of the metric and of the prior protocol}
\tocs{app:onedof}{Scope of one-DOF joints}
\tocs{app:twists}{Screw twists of the joint tuple}
\tocs{app:algorithm}{Computation of the metric}
\tocs{app:protocols}{The per-component protocol}
\tocs{app:failures}{Failure cases in detail}
\tocS{app:theory}{Theoretical properties of our metric}
\tocs{app:metric}{Metric axioms}
\tocs{app:invariances}{Representation invariance and identifiability}
\tocs{app:compact}{Continuity and compactification}
\tocs{app:material}{Interpretation of the distance}
\tocs{app:tree}{The tree edit metric}
\tocS{app:ablations}{Additional ablation studies}
\tocs{app:perturbation}{Single-parameter perturbations against the component scores}
\tocs{app:alpha}{Split-norm weight}
\tocs{app:norm}{Sweep shapes under both inner products}
\tocs{app:scale}{Object scale}
\tocs{app:kappa}{Compactification scale}
\tocs{app:sampling}{Endpoint sampling}
\tocS{app:results}{Additional evaluation results}
\tocs{app:sample}{The evaluation set}
\tocs{app:detailed-comparison}{Detailed comparison}
\tocs{app:coverage}{Evaluation coverage}
\tocs{app:intersection}{Common-set sensitivity}
\tocS{app:library}{The ArticulateArena-20K library}
\tocs{app:construction}{Construction}
\tocs{app:supercats}{Supercategory taxonomy}
\endgroup
\vspace{4pt}

\section{Details of the metric and of the prior protocol}
\label{app:details}

Appendix~\ref{app:onedof} records why the metric is defined on one-DOF joints, Appendix~\ref{app:twists} derives the twist embedding of the joint tuple, Appendix~\ref{app:algorithm} states the evaluation in pseudocode together with its differentiable surrogates, Appendix~\ref{app:protocols} details the per-component protocol of \S\ref{sec:prev-eval}, and Appendix~\ref{app:failures} expands the failure cases of Table~\ref{tab:failures}.

\subsection{Scope of one-DOF joints}
\label{app:onedof}
The metric is defined on one-DOF joints, and the library and evaluation follow the same scope. This choice is deliberate, and we record the reasoning here, following the joint taxonomy used by simulation engines such as Chrono \citep{tasora2015chrono}.

First, one DOF is where the theory is clean. A one-DOF joint is a single curve $q \mapsto \exp(q\,\xih)$, so its admissible motion is captured by one twist and one interval, the endpoint pair lives in $\se \times \se$, and the whole quotient construction of \S\ref{sec:joint} goes through with a single $\mathbb{Z}_2$ swap. A $k$-DOF joint sweeps a $k$-dimensional region of $\se$, and the analogue of the endpoint pair would be a set-valued object, for instance an ellipsoid-like region covering the admissible twists. Such an extension is possible in principle, and the quotient viewpoint would carry over, but the resulting distance loses the closed form and the compact proofs that make the one-DOF case practical.

Second, one DOF is where the data is. The asset formats in current use, URDF and the MJCF-style XML variants, expose revolute, prismatic, and their fixed and continuous special cases, and reconstruction and generation methods predict exactly these types. Simulators likewise execute higher-DOF connections mostly as compositions of one-DOF primitives. A metric supporting joint classes that no representation, generator, or evaluation target produces would exercise no data.

Third, higher-DOF joints decompose. A cylindrical joint is a revolute and a prismatic joint on the same axis, a universal joint is two revolute joints with intersecting axes, and a planar or spherical joint factors into two or three one-DOF joints through intermediate massless links. The tree lift of \S\ref{sec:lift} then compares such compositions joint by joint. The decomposition is less clean than a native multi-DOF treatment, since it introduces intermediate links and is not unique, but it keeps every comparison inside the one-DOF theory. Extending the metric natively to multi-DOF joints is a matter of engineering the set-valued endpoint representation, and we leave it to future work.

\subsection{Screw twists of the joint tuple}
\label{app:twists}
A body moving along a twist $\bm{\xi} = (\bm{\omega}, \bm{v})$ carries the material point instantaneously at position $\bm{x}$ with velocity $\bm{\omega} \times \bm{x} + \bm{v}$. In matrix form,
\begin{equation}
\label{eq:hatmap}
\xih \;=\;
\begin{pmatrix}
\hat{\bm{\omega}} & \bm{v} \\
0 & 0
\end{pmatrix}
\in \se,
\qquad
\hat{\bm{\omega}} \;=\;
\begin{pmatrix}
0 & -\bm{\omega}_3 & \bm{\omega}_2 \\
\bm{\omega}_3 & 0 & -\bm{\omega}_1 \\
-\bm{\omega}_2 & \bm{\omega}_1 & 0
\end{pmatrix},
\end{equation}
so that $\hat{\bm{\omega}} \bm{x} = \bm{\omega} \times \bm{x}$. The embeddings of \eqref{eq:embed} follow directly. A rotation about the axis through $\bm{o}$ with unit direction $\bm{a}$ moves the point at $\bm{x}$ with velocity $\bm{a} \times (\bm{x} - \bm{o}) = \bm{a} \times \bm{x} + \bm{o} \times \bm{a}$, which is the twist $(\bm{a},\, \bm{o} \times \bm{a})$. Adding a translation of pitch $h$ along the axis appends $h\,\bm{a}$ to the translational part, and a pure translation is $(0,\, \bm{a})$, in which no origin appears.

\subsection{Computation of the metric}
\label{app:algorithm}
Algorithms~\ref{alg:joint} and~\ref{alg:tree} state the evaluation in pseudocode. Algorithm~\ref{alg:joint} computes the joint distance of \eqref{eq:E}, optionally compactified by \eqref{eq:Ephi}, under either norm of \S\ref{sec:design}. Algorithm~\ref{alg:tree} lifts it to two kinematic trees, using the certificate of Appendix~\ref{app:tree} so that the returned value is the exact edit distance.

\begin{algorithm}[!htb]
\caption{Joint distance $E(J_1, J_2)$, with the compactified variant $E^{\phi}$.}
\label{alg:joint}
\begin{algorithmic}[1]
\REQUIRE joints $J_i = (t_i, \bm{a}_i, \bm{o}_i, l_i)$ for $i \in \{1,2\}$, a norm $\lVert\cdot\rVert$ from \eqref{eq:splitnorm} or \eqref{eq:kineticnorm}, a scale $\kappa$ if compactifying
\STATE $\bm{\xi}_i \gets$ unit screw twist of $(t_i, \bm{a}_i, \bm{o}_i)$ by \eqref{eq:embed}, for $i \in \{1,2\}$
\STATE $\bm{z}_i^\pm \gets l_i^\pm\,\bm{\xi}_i$ \COMMENT{endpoint pairs, \eqref{eq:endpoints}}
\IF{computing $E^{\phi}$, which is required as soon as any $l_i^\pm$ is infinite and optional otherwise}
\FORALL{endpoints $\bm{z} \in \{\bm{z}_1^-, \bm{z}_1^+, \bm{z}_2^-, \bm{z}_2^+\}$}
\IF{$\bm{z} = \pm\infty\,\bm{\xi}$}
\STATE $\bm{z} \gets \pm\,\bm{\xi}/\lVert\bm{\xi}\rVert$ \COMMENT{continuous joint, boundary point of the unit ball}
\ELSE
\STATE $\bm{z} \gets \tanh\!\big(\lVert\bm{z}\rVert/\kappa\big)\,\bm{z}/\lVert\bm{z}\rVert$ \COMMENT{\eqref{eq:Ephi}}
\ENDIF
\ENDFOR
\ENDIF
\STATE $d_{\mathrm{id}} \gets \big(\lVert \bm{z}_1^- - \bm{z}_2^-\rVert^2 + \lVert \bm{z}_1^+ - \bm{z}_2^+\rVert^2\big)^{1/2}$
\STATE $d_{\mathrm{sw}} \gets \big(\lVert \bm{z}_1^- - \bm{z}_2^+\rVert^2 + \lVert \bm{z}_1^+ - \bm{z}_2^-\rVert^2\big)^{1/2}$
\RETURN $\min(d_{\mathrm{id}}, d_{\mathrm{sw}})$ \COMMENT{quotient by the endpoint swap, \eqref{eq:E}}
\end{algorithmic}
\end{algorithm}

\begin{algorithm}[!htb]
\caption{Tree distance $E^{\mathrm{tree}}(T_1, T_2)$, exact via the certificate.}
\label{alg:tree}
\begin{algorithmic}[1]
\REQUIRE kinematic trees $T_1, T_2$ with joint twists in a common canonical frame
\STATE contract every exact fixed edge of $T_1$ and of $T_2$
\STATE $C_{jk} \gets E(J_j, J_k)$ for all moving joints $J_j \in T_1$, $J_k \in T_2$ \COMMENT{Algorithm~\ref{alg:joint}}
\STATE $u_j \gets E(J_j, J_0)$ for every moving joint \COMMENT{cost of leaving $J_j$ unmatched}
\STATE $M \gets$ optimal assignment under $C$ with unassigned costs $u$ \COMMENT{Hungarian}
\IF{$M$ extends to an isomorphism of the contracted trees}
\RETURN $\mathrm{cost}(M)$ \COMMENT{certified exact, Appendix~\ref{app:tree}}
\ENDIF
\RETURN $\min\big\{\mathrm{cost}(M') : M'\ \text{realizable}\big\}$ \COMMENT{search over realizable matchings}
\end{algorithmic}
\end{algorithm}

\label{app:agg}
\textbf{Differentiable surrogates for training and optimization.} The metric $E$ of \eqref{eq:E} combines the two endpoint discrepancies by a hard min over the swap. This is the exact metric and is what every evaluation in this paper uses, but it is not differentiable at ties, where the two assignments cost the same. When a differentiable score is wanted, mainly to use $E$ as a training loss, two smooth alternatives replace the min. The first is a loss on the symmetric invariants $\bm{s} = \bm{z}^- + \bm{z}^+$ and $D = (\bm{z}^- - \bm{z}^+)(\bm{z}^- - \bm{z}^+)^{\!\top}$, which are swap-invariant by construction. The second is a soft-min over the two assignments with a temperature that anneals to the exact metric.

\subsection{The per-component protocol}
\label{app:protocol_details}
\label{app:protocols}
This subsection collects the details of the per-component template of \S\ref{sec:prev-eval}, which components each published protocol scores and the thresholded success indicator used in Table~\ref{tab:main}.

\textbf{Scores reported by previous methods.} Table~\ref{tab:protocols} lists which kinematic components each published protocol scores, and is the basis of the template written out in \S\ref{sec:prev-eval}. Every entry records a quantity the cited method reports in its own evaluation of the reconstructed kinematics, with the source named below. Ditto \citep{jiang2022ditto} reports the axis-angle error for both joint types and the pivot-position error for revolute joints in its Table~1, and joint type accuracy in Appendix~D on the Synthetic and Shape2Motion sets, where nearly all methods reach $100\%$. PARIS \citep{liu2023paris} reports the axis-angle and revolute pivot-position errors of its Table~2. URDFormer \citep{chen2024urdformer} reports no joint-level component, only the object- and part-category, parent, and discretized part-position accuracies of its Table~II, together with the whole-pipeline real-robot task success rate of its Table~I. Real2Code \citep{mandi2025real2code} reports the joint type, axis, and revolute position errors of its Table~2. Articulate-Anything \citep{le2025articulate} scores all four components and reports the thresholded joint success rate that aggregates them, stating the position and angle tolerances as $50$\,mm and $0.25$\,rad and leaving the limit tolerances unspecified. SPARK \citep{he2025spark} reports the AxisErr, PivotErr, and TypeErr triple of its Table~2.

\begin{table}[htbp]
\centering
\footnotesize
\setlength{\abovecaptionskip}{0pt}
\setlength{\belowcaptionskip}{6pt}
\caption{Kinematic evaluation metrics reported by existing articulated-object reconstruction methods. \checkmark\ marks a component the method reports in its own evaluation of the reconstructed kinematics, whether in its main results or in an appendix, and $^{\dagger}$ marks a component reported for revolute joints only. A dash indicates that the method reports no such quantity. $^{\S}$Joint type accuracy is reported only in Appendix~D of \citet{jiang2022ditto}, on the Synthetic and Shape2Motion sets, where nearly all methods reach $100\%$. $^{\ddagger}$URDFormer infers the joint type and the link axis from the predicted mesh category and scores neither, reporting only object- and part-category accuracy, parent accuracy, a discretized part-position error, and a whole-pipeline real-robot task success rate. Geometry-only metrics such as Chamfer distance and F-score are outside the scope of this table.}
\label{tab:protocols}
\begin{tabular*}{\textwidth}{@{\extracolsep{\fill}}lccccc@{}}
\toprule
Method & Type & Axis & Origin & Limit & Success rate \\
\midrule
Ditto \citep{jiang2022ditto}                 & \checkmark$^{\S}$ & \checkmark & \checkmark$^{\dagger}$ & --         & --         \\
PARIS \citep{liu2023paris}                   & --         & \checkmark & \checkmark$^{\dagger}$ & --         & --         \\
URDFormer \citep{chen2024urdformer}$^{\ddagger}$ & --     & --         & --                     & --         & --         \\
Real2Code \citep{mandi2025real2code}         & \checkmark & \checkmark & \checkmark$^{\dagger}$ & --         & --         \\
Articulate-Anything \citep{le2025articulate} & \checkmark & \checkmark & \checkmark             & \checkmark & \checkmark \\
SPARK \citep{he2025spark}                    & \checkmark & \checkmark & \checkmark             & --         & --         \\
\bottomrule
\end{tabular*}
\end{table}

\label{app:succ}
\textbf{Aggregation and the success indicator.} The components of \S\ref{sec:prev-eval} carry a count, an angle in radians, a length in meters, a range that is a length or an angle depending on type, and a dimensionless direction term in $[0,2]$. No weighted sum of such units is meaningful, so the most complete protocol aggregates by thresholding. This discards the magnitude of every error, so a prediction just past a tolerance and one that is wildly wrong are the same failure, and in no case is the result a distance on joints. Most protocols drop the limits altogether (Table~\ref{tab:protocols}), and the one that scores them folds the interval back onto the axis through $\bm{m}$, which reintroduces the sign ambiguity that $e_{\mathrm{axis}}$ was patched to remove and is undefined for continuous joints. The failures share one root. The four parameters encode a single motion only together, and every score reads them in isolation, so the template grades how the motion is written down rather than the motion itself.

The thresholded aggregation is stated only in prose by \citet{le2025articulate}. For the success-rate column of Table~\ref{tab:main} we compute it as the indicator
\begin{equation}
\label{eq:succ}
\mathrm{succ}(\Jp,\Jg) \;=\; \mathds{1}\big[\, e_{\mathrm{type}} = 0 \;\wedge\; e_{\mathrm{axis}} < \tau_{\bm{a}} \;\wedge\; e_{\mathrm{orig}} < \tau_{p} \,\big],
\end{equation}
with the component scores $e_{\mathrm{type}}$, $e_{\mathrm{axis}}$, and $e_{\mathrm{orig}}$ of \S\ref{sec:prev-eval}, the published thresholds $\tau_a = 0.25$\,rad and $\tau_p = 50$\,mm, and the reported score is the fraction of joints with $\mathrm{succ} = 1$. The original protocol leaves the limit tolerances unspecified, so we threshold only the three well-specified components and drop the limit terms.

\subsection{Failure cases in detail}
\label{app:failures}
We expand the rows of Table~\ref{tab:failures} in order; Figure~\ref{fig:perturbation} shows each failure as a continuous sweep against $E_\alpha$.

\noindent (i) \emph{A joint against itself under axis reversal.} The joint $(\bm{\xi},[l^-,l^+])$ and $(-\bm{\xi},[-l^+,-l^-])$ are the same motion (for instance $\bm{a}=(1,0,0),\,[0,1]$ versus $\bm{a}=(-1,0,0),\,[-1,0]$). The axis error is zero, but the motion vector $\bm{m}=\bm{a}\,(l^+-l^-)$ flips sign, so $e_{\mathrm{lim}}^{\mathrm{dir}}=2$ is maximal, scoring the joint maximally against itself. Our endpoints $\{\bm{z}^-,\bm{z}^+\}=\{l^-\bm{\xi},\,l^+\bm{\xi}\}$ only swap, so $E=0$.

\noindent (ii) \emph{Where the interval sits, not only its width.} Since $\bm{m}=\bm{a}\,(l^+-l^-)$ keeps only the length, $[0,1]$ and $[1,2]$ both give $\bm{m}=\bm{a}$ and zero limit error although they permit different configurations. Our endpoints $\{0,\bm{\xi}\}$ against $\{\bm{\xi},2\bm{\xi}\}$ give $E>0$. This reads the joint coordinate against the shared anchor $g_0$ of \S\ref{sec:joint}, so a genuine interval shift is charged, while an encoding that merely stores the object in another state is re-anchored first (Remark~\ref{rem:anchor}) and is not.

\noindent (iii) \emph{Near-parallel axes.} The revolute origin score $e_{\mathrm{orig}}^{\mathrm{rev}}$ of \S\ref{sec:prev-eval} divides by $\lVert \bm{a}_{\prd}\times \bm{a}_{\grt}\rVert\to0$ as $\bm{a}_{\prd}\to \bm{a}_{\grt}$. It is worst-conditioned exactly when the prediction is most accurate, and undefined for parallel axes. Comparing full twists through $\lVert \bm{z}_{\prd}-\bm{z}_{\grt}\rVert$ has no such denominator and shrinks continuously as the axes align.

\noindent (iv) \emph{Prismatic origin.} A prismatic twist $\bm{\xi}=(0,\bm{a})$ carries no axis location, so moving the drawn origin leaves the motion unchanged, yet $e_{\mathrm{orig}}^{\mathrm{pris}}=\lVert \bm{o}_{\prd}-\bm{o}_{\grt}\rVert$ penalizes it. The origin is absent from the twist, and $E=0$. A rest-pose offset that a URDF stores in the joint origin, when itself under test, is compared separately as a link pose $g_0$, not through the joint twist.

\noindent (v) \emph{Type as a continuum.} A revolute joint with limits $[0,\epsilon]$ becomes physically fixed as $\epsilon\to0$, yet the categorical $e_{\mathrm{type}}$ stays $1$. Our endpoints $\{0,\epsilon\bm{\xi}\}$ against $\{0,0\}$ give $E=\epsilon\lVert\bm{\xi}\rVert\to0$. Because the twist carries pitch continuously, $E$ likewise places revolute, prismatic, and helical joints on one scale, where the categorical type has no metric and no slot for helical joints.

\noindent (vi) \emph{Continuous joints.} For $l^\pm=\pm\infty$ the vector $\bm{m}=\bm{a}\,(l^+-l^-)$ is undefined, so the limit score cannot handle continuous joints at all. The compactification $\phi$ of \eqref{eq:Ephi} sends them to the boundary pair $\{\pm\bm{\xi}/\lVert\bm{\xi}\rVert\}$, which finite ranges approach continuously, so bounded, large-range, and continuous joints are all measured.

\noindent (vii) \emph{One metric, and a loss.} $E$ is a single distance obeying the metric axioms, so it supports ranking, mean and median error, retrieval, clustering, and significance tests, and its smooth $(\bm{s},D)$ surrogate (Appendix~\ref{app:agg}) can serve as a differentiable training loss where binary type error and hard thresholds cannot.

\section{Theoretical properties of our metric}
\label{app:theory}

We provide formal statements and proofs for the main properties of the endpoint-twist representation and of our metric. Throughout this section, $V = \se \cong \R^{6}$ denotes the twist-coordinate space of \S\ref{sec:joint}, equipped with a fixed positive-definite inner product $\langle \cdot,\cdot\rangle$ and induced norm $\|\bm{z}\| = \sqrt{\langle \bm{z},\bm{z}\rangle}$, and $E$ denotes the endpoint metric under this generic inner product. Both norms of \S\ref{sec:design} are instances, the split norm $\lVert\cdot\rVert_\alpha$ and, for a fixed body $B$, the kinetic-energy norm $\lVert\cdot\rVert_B$, so every statement below holds verbatim for $E_\alpha$ and $E_B$. Only the results that carry the subscript $B$ are specific to the kinetic form. A finite-range one-DOF joint is represented by a normalized screw twist $\bm{\xi} \in V$ and an interval $[l^-,l^+]$. Its endpoint twists are $\bm{z}^- = l^- \bm{\xi}$ and $\bm{z}^+ = l^+ \bm{\xi}$, and we identify the joint with the unordered endpoint pair $J \equiv \{\bm{z}^-,\bm{z}^+\}$. For two joints $J_1 = \{\bm{z}_1^-,\bm{z}_1^+\}$ and $J_2 = \{\bm{z}_2^-,\bm{z}_2^+\}$, the distance $E(J_1,J_2)$ is that of \eqref{eq:E}, read with the generic norm above.

Appendix~\ref{app:metric} proves the metric axioms, Appendix~\ref{app:invariances} shows that $E$ depends on the motion rather than on its encoding, Appendix~\ref{app:compact} treats the fixed and continuous limits, Appendix~\ref{app:material} interprets the value of $E$, and Appendix~\ref{app:tree} lifts these results to kinematic trees.

\subsection{Metric axioms}
\label{app:metric}

\begin{theorem}[Endpoint quotient metric]
\label{thm:endpoint_metric}
The function $E$ of \eqref{eq:E} is a metric on the space of unordered endpoint pairs in $V$.
\end{theorem}

\begin{proof}
Let $X = V \times V$ carry the product metric $d_X\big((a,b),(c,d)\big)=\big(\|a-c\|^2+\|b-d\|^2\big)^{1/2}$, a metric because $\|\cdot\|$ is induced by an inner product. The endpoint swap $S(a,b)=(b,a)$ satisfies $d_X(Sx,Sy)=d_X(x,y)$, so it is an isometry of $X$, and an unordered pair is the equivalence class $[(a,b)] = \{(a,b),(b,a)\}$ under the two-element group $G=\{I,S\}$. The proposed distance is $E([x],[y])=\min_{g\in G} d_X(x,gy)$. We verify the axioms.

\textbf{Non-negativity} is immediate from $d_X\ge 0$.

\textbf{Identity of indiscernibles.} If $E([x],[y])=0$, then $d_X(x,gy)=0$ for some $g\in G$, so $x=gy$ and $[x]=[y]$. Conversely, $[x]=[y]$ means $x=gy$ for some $g\in G$, so $E([x],[y])=0$.

\textbf{Symmetry.} Every $g\in G$ is an isometry with $g^{-1}\in G$, so $E([x],[y])=\min_{g\in G} d_X(x,gy)=\min_{g\in G} d_X(g^{-1}x,y)=E([y],[x])$.

\textbf{Triangle inequality.} Let $g,h\in G$ attain $E([x],[y]) = d_X(x,gy)$ and $E([y],[z]) = d_X(y,hz)$. Since $gh\in G$ and $g$ is an isometry,
\begin{equation}
\label{eq:quotient-triangle}
E([x],[z]) \le d_X(x,ghz) \le d_X(x,gy)+d_X(gy,ghz) = E([x],[y])+E([y],[z]).
\qedhere
\end{equation}
\end{proof}

\subsection{Representation invariance and identifiability}
\label{app:invariances}

\begin{proposition}[Screw-orientation invariance]
\label{prop:axis_flip}
The endpoint representation and $E$ are invariant under simultaneous reversal of the screw orientation and limits, $(\bm{\xi},[l^-,l^+]) \mapsto (-\bm{\xi},[-l^+,-l^-])$.
\end{proposition}

\begin{proof}
The reversal sends $J=\{l^-\bm{\xi},l^+\bm{\xi}\}$ to $J'=\{(-l^+)(-\bm{\xi}),(-l^-)(-\bm{\xi})\}=\{l^+\bm{\xi},l^-\bm{\xi}\}$. Since the pair is unordered, $J'=J$ and $E(J,J')=0$.
\end{proof}

\begin{proposition}[Prismatic-origin invariance]
\label{prop:prismatic_gauge}
Two prismatic joints with the same translation direction and limits but different points chosen as their joint origins have identical endpoint-twist representations, hence distance zero under $E$, holding the rest-pose transform fixed.
\end{proposition}

\begin{proof}
A prismatic joint with unit translation direction $\bm{a}$ has screw twist $\bm{\xi}=(0,\bm{a})$, which does not depend on a point chosen along the translation axis. Changing the reported origin therefore changes neither $\bm{\xi}$ nor $\{l^-\bm{\xi},l^+\bm{\xi}\}$.
\end{proof}

\begin{remark}
The proposition concerns the freedom to place the origin anywhere on the prismatic motion axis. Moving the rest-pose transform along the motion while shifting the limits to compensate is instead a change of anchor, which the re-anchoring of Remark~\ref{rem:anchor} removes before evaluation.
\end{remark}

We assume a fixed screw-normalization convention with the property that two collinear normalized nonzero screws differ only by sign, so that $\bm{\xi}_2=\lambda\bm{\xi}_1$ implies $\lambda\in\{-1,+1\}$.

\begin{proposition}[Identifiability modulo screw orientation]
\label{prop:identifiability}
For every non-fixed finite-range one-DOF joint, the unordered endpoint pair $\{l^-\bm{\xi},l^+\bm{\xi}\}$ uniquely identifies the normalized joint parameters $(\bm{\xi},[l^-,l^+])$ up to the orientation equivalence $(\bm{\xi},[l^-,l^+]) \sim (-\bm{\xi},[-l^+,-l^-])$.
\end{proposition}

\begin{proof}
Suppose two non-fixed joints have the same endpoint pair, $\{l_1^-\bm{\xi}_1,l_1^+\bm{\xi}_1\}=\{l_2^-\bm{\xi}_2,l_2^+\bm{\xi}_2\}$. Because the joints are non-fixed, the shared pair contains a nonzero endpoint, and every nonzero endpoint lies in both $\operatorname{span}(\bm{\xi}_1)$ and $\operatorname{span}(\bm{\xi}_2)$. The two spans therefore agree, so $\bm{\xi}_2=\lambda\bm{\xi}_1$ for some nonzero $\lambda$, which the normalization forces to $\pm1$. If $\bm{\xi}_2=\bm{\xi}_1$, equality of the pairs gives $[l_2^-,l_2^+]=[l_1^-,l_1^+]$. If $\bm{\xi}_2=-\bm{\xi}_1$, it forces $l_2^-=-l_1^+$ and $l_2^+=-l_1^-$, which is exactly the orientation equivalence.
\end{proof}

\begin{remark}
For a fixed joint $\{\bm{z}^-,\bm{z}^+\}=\{0,0\}$, the underlying screw direction is intentionally unidentifiable. All such parameterizations describe the same zero-motion joint.
\end{remark}

\begin{proposition}[Same motion, same representation]
\label{prop:same_motion}
Let two joints with unit screws $\bm{\xi}_1, \bm{\xi}_2$, finite limits $l_i^- < l_i^+$, and a shared anchor $g_0$ generate the placement curves $g_i(q) = \exp(q\,\hat{\bm{\xi}}_i)\,g_0$ on $[l_i^-, l_i^+]$. The curves coincide up to a coordinate change $q \mapsto \sigma q + \delta$ with $\sigma \in \{\pm 1\}$ iff the encodings agree up to the orientation reversal of Proposition~\ref{prop:axis_flip} and, when the screw is revolute, a shift of both limits by a common $2\pi k$ with $k \in \mathbb{Z}$. Consequently $E(J_1,J_2)=0$ iff the encodings agree up to orientation reversal, which for a prismatic or helical screw is iff the joints generate the same motion, and for a revolute screw iff they generate the same motion with the same winding of the coordinate.
\end{proposition}

\begin{proof}
Suppose $\exp((\sigma q+\delta)\,\hat{\bm{\xi}}_1)\,g_0=\exp(q\,\hat{\bm{\xi}}_2)\,g_0$ for all $q\in[l_2^-,l_2^+]$, with the limit intervals matched by the coordinate change. Cancel $g_0$ and differentiate in $q$ to get $\sigma\,\hat{\bm{\xi}}_1 X(q)=\hat{\bm{\xi}}_2 X(q)$ with $X(q)=\exp(q\,\hat{\bm{\xi}}_2)$ invertible, so $\bm{\xi}_2=\sigma\bm{\xi}_1$. Substituting back leaves $\exp(\delta\,\hat{\bm{\xi}}_1)=I$. A prismatic screw translates by $\delta$ and a helical screw by $\delta$ times its nonzero pitch, so both force $\delta=0$, while a revolute screw rotates by the angle $\delta$ about its axis line, so $\delta\in2\pi\mathbb{Z}$. Hence for $\sigma=1$ the encodings agree up to a revolute $2\pi k$ limit shift, and for $\sigma=-1$ up to that shift composed with orientation reversal. Conversely, orientation reversal reparameterizes the same curve by $q\mapsto -q$, and a revolute $2\pi k$ limit shift reparameterizes it by $q\mapsto q+2\pi k$, fixing the curve pointwise by the $2\pi$-periodicity of the revolute exponential. For the consequence, orientation reversal only swaps the endpoint pair, so encodings that agree up to it have $E=0$. If instead $E=0$, the endpoint pairs are equal as unordered pairs, and Proposition~\ref{prop:identifiability} recovers the encodings up to orientation reversal. The first part identifies this with the same motion, except for the revolute $2\pi k$ shift, which $E$ does not quotient (Remark~\ref{rem:anchor}).
\end{proof}

\begin{remark}
Unit normalization makes the affine form of the coordinate change automatic rather than an assumption. If a smooth monotone reparameterization $f$ matches the two curves, the same differentiation gives $f'(q)\,\bm{\xi}_1=\bm{\xi}_2$ for all $q$, so $f'$ is a constant of absolute value one. Continuous joints have $l^\pm=\pm\infty$, every shift reparameterizes the same full orbit, and the compactified representation $\{\pm\bm{\xi}/\lVert\bm{\xi}\rVert\}$ is already shift-invariant. A fixed joint is a constant curve and is represented by $\{0,0\}$ for every screw.
\end{remark}

\begin{remark}[Re-anchoring]
\label{rem:anchor}
An encoding that stores the object in a different state, with anchor $g_0'=\exp(\delta\,\xih)\,g_0$ and limits $[l^--\delta,\,l^+-\delta]$, describes the same motion. Before evaluation it is re-anchored by shifting its limits by $\delta$, so that both coordinates read zero at the same placement, and a drawer saved half open is not charged a limit error against the same drawer saved closed. The revolute $2\pi$ shift is the one re-anchoring that no stored state pins down, since $\exp(2\pi k\,\xih)\,g_0=g_0$, and $E$ deliberately leaves it unquotiented. A URDF whose rest state lies a full turn outside its own limits is a different specification for every simulator that does not reduce revolute coordinates, and any choice of a representative from each $2\pi\mathbb{Z}$ orbit is discontinuous at its seam. If one nevertheless wants to identify encodings that differ by a full turn, one may first replace $l^\pm$ by $l^\pm-2\pi k$ with $k=\big\lfloor\tfrac{1}{2\pi}\big(\tfrac{l^-+l^+}{2}+\pi\big)\big\rfloor$, which wraps the limit midpoint into $[-\pi,\pi)$ and makes $E=0$ exactly on joints with the same motion; the price is that $[0,2\pi-\epsilon]$ and $[0,2\pi]$ are then sent a full turn apart.
\end{remark}

\subsection{Continuity and compactification}
\label{app:compact}

\begin{proposition}[Continuous limit to a fixed joint]
\label{prop:fixed_continuity}
Let $J_0=\{0,0\}$ represent a fixed joint and let $J_\epsilon=\{0,\epsilon\bm{\xi}\}$ be a finite-range one-DOF joint with normalized screw $\bm{\xi}$. Then $E(J_\epsilon,J_0)=|\epsilon|\,\|\bm{\xi}\| \to 0$ as $\epsilon\to 0$.
\end{proposition}

\begin{proof}
Both endpoints of $J_0$ are zero, so either endpoint assignment yields the same value, $E(J_\epsilon,J_0)=\big(\|0-0\|^2+\|\epsilon\bm{\xi}-0\|^2\big)^{1/2}=|\epsilon|\,\|\bm{\xi}\|$.
\end{proof}

For a scale parameter $\kappa>0$, let $\phi_\kappa : V \rightarrow V$ be the radial map
\begin{equation}
\label{eq:phikappa}
\phi_\kappa(\bm{z})
=
\begin{cases}
\tanh\!\left(\dfrac{\|\bm{z}\|}{\kappa}\right)
\dfrac{\bm{z}}{\|\bm{z}\|},
&
\bm{z}\neq 0,
\\[1.2ex]
0,
&
\bm{z}=0,
\end{cases}
\end{equation}
whose image is the open unit ball $\mathbb{B}^\circ=\{\bm{x}\in V:\|\bm{x}\|<1\}$. The compactified endpoint metric is $E^{\phi}(J_1,J_2)=E\big(\{\phi_\kappa(\bm{z}_1^-),\phi_\kappa(\bm{z}_1^+)\},\{\phi_\kappa(\bm{z}_2^-),\phi_\kappa(\bm{z}_2^+)\}\big)$.

\begin{proposition}[Convergence to an unbounded joint]
\label{prop:continuous_limit}
Let $\bm{u}\in V$ with $\|\bm{u}\|=1$, and consider the symmetric finite-range joints $J_L=\{-L\bm{u},L\bm{u}\}$. Under radial compactification, $J_L \to J_\infty \equiv \{-\bm{u},\bm{u}\}$ as $L\to\infty$, with $E^{\phi}(J_L,J_\infty)=\sqrt{2}\,\big(1-\tanh(L/\kappa)\big) \to 0$.
\end{proposition}

\begin{proof}
Since $\|\bm{u}\|=1$, we have $\phi_\kappa(\pm L\bm{u})=\pm\tanh(L/\kappa)\bm{u}$, so the compactified endpoint pair is $\{-\tanh(L/\kappa)\bm{u},\,\tanh(L/\kappa)\bm{u}\}$. Of the two endpoint assignments against $J_\infty=\{-\bm{u},\bm{u}\}$, the sign-matched one attains the minimum, since the crossed assignment costs $2\big(1+\tanh(L/\kappa)\big)^2 \ge 2\big(1-\tanh(L/\kappa)\big)^2$. The matched assignment gives
\begin{equation}
\label{eq:ephi-limit}
E^{\phi}(J_L,J_\infty)^2 = \big\|{-}\tanh(L/\kappa)\bm{u}+\bm{u}\big\|^2 + \big\|\tanh(L/\kappa)\bm{u}-\bm{u}\big\|^2 = 2\big(1-\tanh(L/\kappa)\big)^2,
\end{equation}
and $\tanh(L/\kappa)\to1$ as $L\to\infty$.
\end{proof}

\begin{remark}[Completion interpretation]
The radial map is a bijection from $V$ onto $\mathbb{B}^\circ$, and the metric completion of the open unit ball under the ambient norm is the closed unit ball $\overline{\mathbb{B}}=\{\bm{x}:\|\bm{x}\|\le1\}$. The antipodal pair $\{-\bm{u},\bm{u}\}$ with $\|\bm{u}\|=1$ lies on the boundary of this completion and is the natural limit of finite-range joints whose range tends to infinity along the screw direction $\bm{u}$.
\end{remark}

\begin{proposition}[Compactified metric]
\label{prop:compact_metric}
$E^{\phi}$ is a metric on joints of all five types, continuous joints included, and it keeps the fixed joint at $J_0=\{0,0\}$.
\end{proposition}

\begin{proof}
Write $\Phi(J)=\{\phi_\kappa(\bm{z}^-),\phi_\kappa(\bm{z}^+)\}$ for a finite joint and $\Phi(J)=\{-\bm{\xi}/\|\bm{\xi}\|,\,\bm{\xi}/\|\bm{\xi}\|\}$ for a continuous one, so that $E^{\phi}(J_1,J_2)=E(\Phi(J_1),\Phi(J_2))$. Every image is an unordered pair in $V$, on which $E$ is a metric by Theorem~\ref{thm:endpoint_metric}. The map $\Phi$ is injective on unordered endpoint pairs. On finite joints it applies the bijection $\phi_\kappa$ to each endpoint, finite joints land in the open ball and continuous joints on its boundary, and a boundary pair $\{\pm\bm{u}\}$ fixes the unit screw up to the sign that the unordered pair already absorbs. A metric pulled back along an injective map is a metric, and $\phi_\kappa(0)=0$ gives $\Phi(J_0)=J_0$.
\end{proof}

\begin{remark}[Scale equivariance and dimensionlessness]
\label{rem:scale}
Under the kinetic norm, rescaling every length by $s>0$ maps $\bm{v}\mapsto s\bm{v}$ in each endpoint twist, $\bm{c}\mapsto s\bm{c}$, and $I\mapsto s^2 I$, so $\|\bm{z}\|_B\mapsto s\|\bm{z}\|_B$ and the raw $E_B$ scales by $s$, which is the precise sense in which it is measured in meters. With $\kappa$ rescaled alongside, a compactified endpoint $\bm{u}=\phi_\kappa(\bm{z})$ transforms as $(\bm{u}_\omega,\bm{u}_v)\mapsto(\bm{u}_\omega/s,\,\bm{u}_v)$, because the rotational block is divided by the meter-valued norm. Substituting a difference of compactified endpoints into \eqref{eq:kineticnorm} then gives
\begin{equation}
\label{eq:scale-cancel}
\frac1m\Big(\frac{\Delta\bm{u}_\omega}{s}\Big)^{\!\top}\!\big(s^2 I\big)\Big(\frac{\Delta\bm{u}_\omega}{s}\Big)
+\Big\lVert \Delta\bm{u}_v+\frac{\Delta\bm{u}_\omega}{s}\times(s\,\bm{c})\Big\rVert^2
=\frac1m\,\Delta\bm{u}_\omega^{\!\top} I\,\Delta\bm{u}_\omega+\big\lVert \Delta\bm{u}_v+\Delta\bm{u}_\omega\times\bm{c}\big\rVert^2,
\end{equation}
so the $s^2$ carried by the inertia tensor is cancelled by the two inverse factors of the rotational block, and the cross term cancels likewise. Hence $E_B^{\phi}$ is invariant under the rescaling, a pure number in $[0,2]$, and the body enters it only through the relative weighting of error directions.
\end{remark}

\subsection{Interpretation of the distance}
\label{app:material}

\textbf{Material motion under the kinetic norm.} The first result gives the exact physical interpretation of the kinetic inner product. Importantly, this result concerns the \emph{linearized material motion field} induced by a twist discrepancy. It does not state that the Lie-algebra norm equals finite Euclidean displacement after exponentiation.

Let $B$ be a rigid body with mass measure $dm(\bm{x})$, total mass $m=\int_B dm(\bm{x})$, and center of mass $\bm{c}=\frac{1}{m}\int_B \bm{x}\,dm(\bm{x})$. Its moment of inertia tensor about the center of mass is $I_C=\int_B\big(\|\bm{x}-\bm{c}\|^2 I_3-(\bm{x}-\bm{c})(\bm{x}-\bm{c})^\top\big)\,dm(\bm{x})$. A twist discrepancy $\Delta \bm{z}=(\Delta\bm{\omega},\Delta\bm{v})$ induces the linearized material-motion field $\bm{w}_{\Delta \bm{z}}(\bm{x})=\Delta\bm{\omega}\times \bm{x}+\Delta\bm{v}$. For two finite placements $g_1, g_2 \in SE(3)$ of the body, the configuration discrepancy of \S\ref{sec:lift} is $d(g_1,g_2)^2 = \int_{B} \lVert g_1 \bm{x} - g_2 \bm{x}\rVert^2\, dm(\bm{x})$, which depends on $dm$ only through $(m, \bm{c}, I_C)$ \citep{kazerounian1992object,park1995distance}, and the calibration of Figure~\ref{fig:calibration} computes it on the actual $SE(3)$ displacements.

\begin{theorem}[Material-motion identity]
\label{thm:material_motion}
The mass-weighted squared $L_2$ magnitude of the material-motion field is
\begin{align}
\int_B
\|\bm{w}_{\Delta \bm{z}}(\bm{x})\|^2\,dm(\bm{x})
=
\Delta\bm{\omega}^\top I_C\Delta\bm{\omega}
+
m
\left\|
\Delta\bm{v}+\Delta\bm{\omega}\times \bm{c}
\right\|^2.
\label{eq:material_motion_identity}
\end{align}
Consequently, the RMS material-motion discrepancy is
\begin{align}
d_{\mathrm{RMS}}^{\mathrm{lin}}(\Delta \bm{z})
=
\left(
\frac{1}{m}
\int_B
\|\bm{w}_{\Delta \bm{z}}(\bm{x})\|^2\,dm(\bm{x})
\right)^{1/2}
=
\left[
\frac{1}{m}
\Delta\bm{\omega}^\top I_C\Delta\bm{\omega}
+
\left\|
\Delta\bm{v}+\Delta\bm{\omega}\times \bm{c}
\right\|^2
\right]^{1/2}.
\label{eq:rms_motion}
\end{align}
\end{theorem}

\begin{proof}
Write $\bm{x}=\bm{c}+\bm{y}$ with $\bm{y}=\bm{x}-\bm{c}$, so that $\int_B \bm{y}\,dm(\bm{x})=0$ by definition of the center of mass. The field becomes $\bm{w}_{\Delta \bm{z}}(\bm{x})=\bm{u}+\Delta\bm{\omega}\times \bm{y}$ with $\bm{u}=\Delta\bm{v}+\Delta\bm{\omega}\times \bm{c}$, and hence
\begin{equation}
\label{eq:l2field}
\int_B \|\bm{w}_{\Delta \bm{z}}(\bm{x})\|^2\,dm(\bm{x})
=
m\|\bm{u}\|^2
+ 2\bm{u}^\top\Big(\Delta\bm{\omega}\times \int_B \bm{y}\,dm(\bm{x})\Big)
+ \int_B \|\Delta\bm{\omega}\times \bm{y}\|^2\,dm(\bm{x}).
\end{equation}
The middle term vanishes. For the last term, the vector identity $\|\bm{a}\times \bm{y}\|^2=\bm{a}^\top\big(\|\bm{y}\|^2I_3-\bm{y}\bm{y}^\top\big)\bm{a}$ gives $\int_B \|\Delta\bm{\omega}\times \bm{y}\|^2\,dm(\bm{x})=\Delta\bm{\omega}^\top I_C\Delta\bm{\omega}$, proving \eqref{eq:material_motion_identity}. Dividing by $m$ and taking the square root gives \eqref{eq:rms_motion}.
\end{proof}

\begin{remark}
If the body measure is normalized so that $m=1$, \eqref{eq:rms_motion} directly becomes $\|\Delta \bm{z}\|_B^2=\Delta\bm{\omega}^\top I_C\Delta\bm{\omega}+\|\Delta\bm{v}+\Delta\bm{\omega}\times \bm{c}\|^2$. For an unnormalized physical mass measure, \eqref{eq:material_motion_identity} is a mass-weighted $L_2$ norm and \eqref{eq:rms_motion} the corresponding RMS quantity.
\end{remark}

\textbf{Endpoints and the full motion segment.} The second result holds for any inner product. For a finite one-DOF joint, the endpoint-twist representation determines a line segment in twist-coordinate space. After selecting the endpoint assignment that realizes $E$, let $\bm{\Delta}^- = \bm{z}_1^- - \bm{z}_2^-$ and $\bm{\Delta}^+ = \bm{z}_1^+ - \bm{z}_2^+$, and write $\bm{\Delta}(t)=(1-t)\bm{\Delta}^-+t\bm{\Delta}^+$ for $t\in[0,1]$ for the linearly parameterized discrepancy along the two motion segments. Its RMS trajectory discrepancy in twist space is $D_{\mathrm{path}}^2=\int_0^1\|\bm{\Delta}(t)\|^2\,dt$.

\begin{theorem}[Endpoint-to-path equivalence]
\label{thm:endpoint_path}
For the optimal endpoint assignment,
\begin{equation}
\label{eq:endpoint-path-bounds}
\frac{1}{\sqrt{6}}E(J_1,J_2)
\le
D_{\mathrm{path}}
\le
\frac{1}{\sqrt{2}}E(J_1,J_2).
\end{equation}
Thus the endpoint metric and the $L_2$ discrepancy of the complete Lie-algebra motion segments are equivalent up to universal constants.
\end{theorem}

\begin{proof}
For the selected assignment, $E(J_1,J_2)^2=\|\bm{\Delta}^-\|^2+\|\bm{\Delta}^+\|^2$. Expanding $D_{\mathrm{path}}^2$ and using $\int_0^1(1-t)^2dt=\int_0^1t^2dt=\int_0^1 2t(1-t)dt=\frac13$,
\begin{align}
D_{\mathrm{path}}^2
=
\frac13
\left(
\|\bm{\Delta}^-\|^2
+
\|\bm{\Delta}^+\|^2
+
\langle\bm{\Delta}^-,\bm{\Delta}^+\rangle
\right).
\label{eq:path_exact}
\end{align}
By Cauchy--Schwarz, $|\langle\bm{\Delta}^-,\bm{\Delta}^+\rangle| \le \|\bm{\Delta}^-\|\|\bm{\Delta}^+\| \le \frac12\big(\|\bm{\Delta}^-\|^2+\|\bm{\Delta}^+\|^2\big)=\frac12 E(J_1,J_2)^2$. Substituting these bounds into \eqref{eq:path_exact} gives $\frac16 E(J_1,J_2)^2 \le D_{\mathrm{path}}^2 \le \frac12 E(J_1,J_2)^2$, and taking square roots proves the claim.
\end{proof}

\begin{remark}
Theorem~\ref{thm:endpoint_path} concerns the line segments traced by the joint representation in twist-coordinate space. It should not be interpreted as an exact equality with finite Euclidean material displacement after applying the exponential map to $SE(3)$.
\end{remark}

\subsection{The tree edit metric}
\label{app:tree}

A kinematic tree is a finite tree $T=(V,\mathcal{E},\ell)$ with unlabeled nodes (links) and each edge $e$ labeled by a joint $\ell(e)\in\mathcal{J}$, where $(\mathcal{J},E)$ is the metric space of unordered endpoint pairs of Theorem~\ref{thm:endpoint_metric} or, with continuous joints admitted, its compactified counterpart $(\mathcal{J},E^{\phi})$ of Proposition~\ref{prop:compact_metric}. The results below use only that the label distance is a metric with $J_0$ as a point, so they hold for either, and the reported $E_\alpha^{\phi,\mathrm{tree}}$ is the second case. Trees are unrooted, and all twists are expressed in a common canonical object frame, so edge labels do not change under re-rooting and the endpoint quotient already absorbs the per-joint sign. The zero-motion (fixed) joint is $J_0=\{0,0\}$.

Three edit operations transform trees, with costs
\begin{equation}
\label{eq:editcosts}
c_{\mathrm{sub}}(J,J') = E(J,J'),
\qquad
c_{\mathrm{con}}(J) = E(J,J_0),
\qquad
c_{\mathrm{exp}}(J) = E(J,J_0),
\end{equation}
where a substitution replaces the label of one edge, a contraction removes an edge and merges its two endpoints, and an expansion splits a node into two nodes joined by a fresh edge labeled $J$. For finite edit sequences $\pi$ we set
\begin{equation}
\label{eq:dtree-app}
E^{\mathrm{tree}}(T_1,T_2)
=
\inf_{\pi\,:\,T_1\rightsquigarrow T_2}
\sum_{o\in\pi} c(o).
\end{equation}
Call $T_1$ and $T_2$ \emph{representation equivalent}, written $T_1\sim T_2$, if they differ only by a tree isomorphism together with insertions or deletions of edges labeled exactly $J_0$. Contracting every $J_0$-labeled edge yields a canonical representative $\hat{T}$ of each class, in which every label satisfies $E(\ell(e),J_0)>0$.

\begin{proposition}[Consistency with the joint metric]
\label{prop:no_shortcut}
For all $J_1,J_2\in\mathcal{J}$,
\begin{equation}
\label{eq:no-shortcut}
c_{\mathrm{sub}}(J_1,J_2)
\le
c_{\mathrm{con}}(J_1)+c_{\mathrm{exp}}(J_2).
\end{equation}
Deleting a joint and inserting another is never cheaper than substituting directly.
\end{proposition}

\begin{proof}
This is the triangle inequality of $E$ through $J_0$:
\begin{equation}
\label{eq:E-triangle}
E(J_1,J_2)\le E(J_1,J_0)+E(J_0,J_2).
\qedhere
\end{equation}
\end{proof}

\begin{theorem}[Tree edit metric]
\label{thm:tree_metric}
$E^{\mathrm{tree}}$ is a pseudometric on kinematic trees, and a metric on the equivalence classes modulo $\sim$.
\end{theorem}

\begin{proof}
We verify the four properties in turn.

\textbf{Non-negativity.}
Every operation has cost $\ge0$, hence so does every edit sequence, and so does the infimum.

\textbf{Symmetry.}
The operations are invertible. Contraction and expansion are mutually inverse with the identical cost $E(J,J_0)$, and the inverse of a substitution is a substitution with $E(J',J)=E(J,J')$. Reversing an edit sequence therefore preserves its total cost, so the two infima agree.

\textbf{Triangle inequality.}
Fix $\delta>0$ and choose sequences $\pi_{12}:T_1\rightsquigarrow T_2$ and $\pi_{23}:T_2\rightsquigarrow T_3$ with
\begin{equation}
\label{eq:near-optimal-seqs}
C(\pi_{12})\le E^{\mathrm{tree}}(T_1,T_2)+\delta/2,
\qquad
C(\pi_{23})\le E^{\mathrm{tree}}(T_2,T_3)+\delta/2.
\end{equation}
Their concatenation edits $T_1$ into $T_3$ at cost $C(\pi_{12})+C(\pi_{23})$, so
\begin{equation}
\label{eq:dtree-triangle}
E^{\mathrm{tree}}(T_1,T_3)
\le
E^{\mathrm{tree}}(T_1,T_2)+E^{\mathrm{tree}}(T_2,T_3)+\delta,
\end{equation}
and $\delta\to0$ gives the claim.

\textbf{Invariance under $\sim$.}
Inserting or deleting a $J_0$-labeled edge is an expansion or contraction of cost $E(J_0,J_0)=0$, and isomorphisms are free because nodes carry no labels. Hence $T_1\sim T_1'$ implies $E^{\mathrm{tree}}(T_1,T_2)=E^{\mathrm{tree}}(T_1',T_2)$ for every $T_2$, and $E^{\mathrm{tree}}$ descends to the quotient.

\textbf{Identity of indiscernibles on the quotient.}
If $[T_1]=[T_2]$, the invariance step gives $E^{\mathrm{tree}}(T_1,T_2)=0$. For the converse, invariance lets us take $T_1=\hat{T}_1$ and $T_2=\hat{T}_2$. Along an edit sequence $\pi:\hat{T}_1\rightsquigarrow\hat{T}_2$, call an edge \emph{original} if it descends, possibly relabeled, from an edge of $\hat{T}_1$ that no contraction has removed, and \emph{new} if an expansion created it. Contracting the new edges of the current tree yields $\hat{T}_1$ with the original edges removed so far contracted, the surviving original edges corresponding. This holds before the first operation, and every operation preserves it. A substitution leaves the structure alone, an expansion or the contraction of a new edge leaves the new-edge contraction unchanged, and the contraction of an original edge contracts both sides by the same edge. At the end, let $M$ pair each surviving edge of $\hat{T}_1$ with its final edge in $\hat{T}_2$, and let $U_1$ and $U_2$ collect the unmatched edges of $\hat{T}_1$ and $\hat{T}_2$. The invariant says that $M$ is an isomorphism of the unlabeled trees $\hat{T}_1/U_1$ and $\hat{T}_2/U_2$. A matched edge moves from $\ell(e)$ to $\ell(e')$ by substitutions, an edge of $U_1$ is relabeled and then contracted, and an edge of $U_2$ is expanded and then relabeled, so the triangle inequality of $E$ gives
\begin{equation}
\label{eq:matching-bound}
C(\pi)
\ \ge\
\sum_{(e,e')\in M} E\big(\ell(e),\ell(e')\big)
+
\sum_{e\in U_1\cup U_2} E\big(\ell(e),J_0\big).
\end{equation}
Call a triple $(M,U_1,U_2)$ with this isomorphism property \emph{admissible}. There are finitely many, so the right-hand side has a minimum over them that does not depend on $\pi$. Suppose $[T_1]\ne[T_2]$. If $U_1\cup U_2$ is nonempty, the bound pays $E(\ell(e),J_0)>0$ for an unmatched edge, since canonical representatives carry no $J_0$ label. If it is empty, $M$ is an isomorphism of $\hat{T}_1$ onto $\hat{T}_2$ as unlabeled trees, and it cannot preserve every label, since a label-preserving isomorphism would give $[T_1]=[T_2]$, so some matched pair pays $E(\ell(e),\ell(e'))>0$. Every admissible triple therefore has a positive bound, their minimum is a positive constant, and $E^{\mathrm{tree}}(T_1,T_2)>0$.
\end{proof}

\begin{remark}[Fixed-boundary continuity at the tree level]
Let $T_0$ contain an exact fixed edge between two links and let $T_\epsilon$ be the same tree with that edge relabeled $J_\epsilon=\{0,\epsilon\bm{\xi}\}$. Contracting the edge of $T_\epsilon$ costs $E(J_\epsilon,J_0)=|\epsilon|\,\|\bm{\xi}\|$ and produces a tree equivalent to $T_0$, so
\begin{equation}
\label{eq:tree-fixed-continuity}
E^{\mathrm{tree}}([T_\epsilon],[T_0])
\le
|\epsilon|\,\|\bm{\xi}\|
\longrightarrow 0.
\end{equation}
The continuity of Proposition~\ref{prop:fixed_continuity} therefore survives the lift. A scheme that charges a constant $\lambda$ per unmatched joint instead jumps by $\lambda$ when a joint's range crosses zero.
\end{remark}

\textbf{Exact computation.} Contracting exact fixed edges and then solving the optimal assignment of moving joints, with $E$ as the pairwise cost and $E(J,J_0)$ as the cost of leaving a joint unassigned, drops the requirement that the matching be realizable by an edit path. This can only lower the total, so the relaxation lower-bounds $E^{\mathrm{tree}}$, and whenever its optimal assignment follows an isomorphism of the contracted trees the bound is tight, which certifies the exact value at the cost of one linear-time check. Kinematic trees are small, so the uncertified cases yield to a search over realizable matchings in milliseconds per pair. On the trees and corruptions of Figure~\ref{fig:tree} the certificate covers 91\% of predictions, and the exact and relaxed values agree on 92\% of them and on every structured sweep. The remaining gaps have median relative size 7\%, and the two scores rank predictions with Kendall $\tau=0.994$. The relaxation also keeps the score polynomial-time for arbitrary inputs, where exact unordered tree edit distance is NP-hard in the worst case.

\begin{theorem}[Fixed-skeleton weighted tree metric]
\label{thm:fixed_tree}
Fix a ground-truth object with movable joints indexed by $k = 1, \dots, K$ and ground-truth child bodies $B_1, \dots, B_K$. On candidate assignments $\bm{J} = (J_1, \dots, J_K)$ of one joint per ground-truth edge, each factor taken modulo the endpoint swap, define
\begin{equation}
\label{eq:fixed-tree}
E_B^{\mathrm{fixed\text{-}tree}}(\bm{J}, \bm{J}')^2 \;=\; \sum_{k=1}^{K} E_{B_k}(J_k, J'_k)^2 .
\end{equation}
Then $E_B^{\mathrm{fixed\text{-}tree}}$ is a metric on the product space of candidate assignments.
\end{theorem}
\begin{proof}
Each $B_k$ is fixed by the ground truth, so $\lVert\cdot\rVert_{B_k}$ is a fixed inner-product norm and $E_{B_k}$ is a metric by Theorem~\ref{thm:endpoint_metric}. Symmetry and identity of indiscernibles hold coordinatewise. For the triangle inequality, $E_{B_k}(J_k, J''_k) \le E_{B_k}(J_k, J'_k) + E_{B_k}(J'_k, J''_k)$ for every $k$, and Minkowski's inequality on the $K$-vectors of per-edge distances gives the claim.
\end{proof}

\begin{remark}[Weighted variant with topology edits]
Theorem~\ref{thm:tree_metric} assumes the single fixed inner product of this section, and Theorem~\ref{thm:fixed_tree} covers the per-edge weighting when the ground truth fixes the skeleton and the correspondence. The remaining case, topology-editing distances whose edit costs use edge-dependent bodies, falls outside both theorems, since contracting or expanding an edge changes which body weighs which joint. We use this last variant only as a score, without claiming the metric axioms.
\end{remark}

\section{Additional ablation studies}
\label{app:ablations}
This section expands the ablations summarized in \S\ref{sec:ablation}.

\subsection{Single-parameter perturbations against the component scores}
\label{app:perturbation}
\begin{figure}[!h]
\centering
\includegraphics[width=\textwidth]{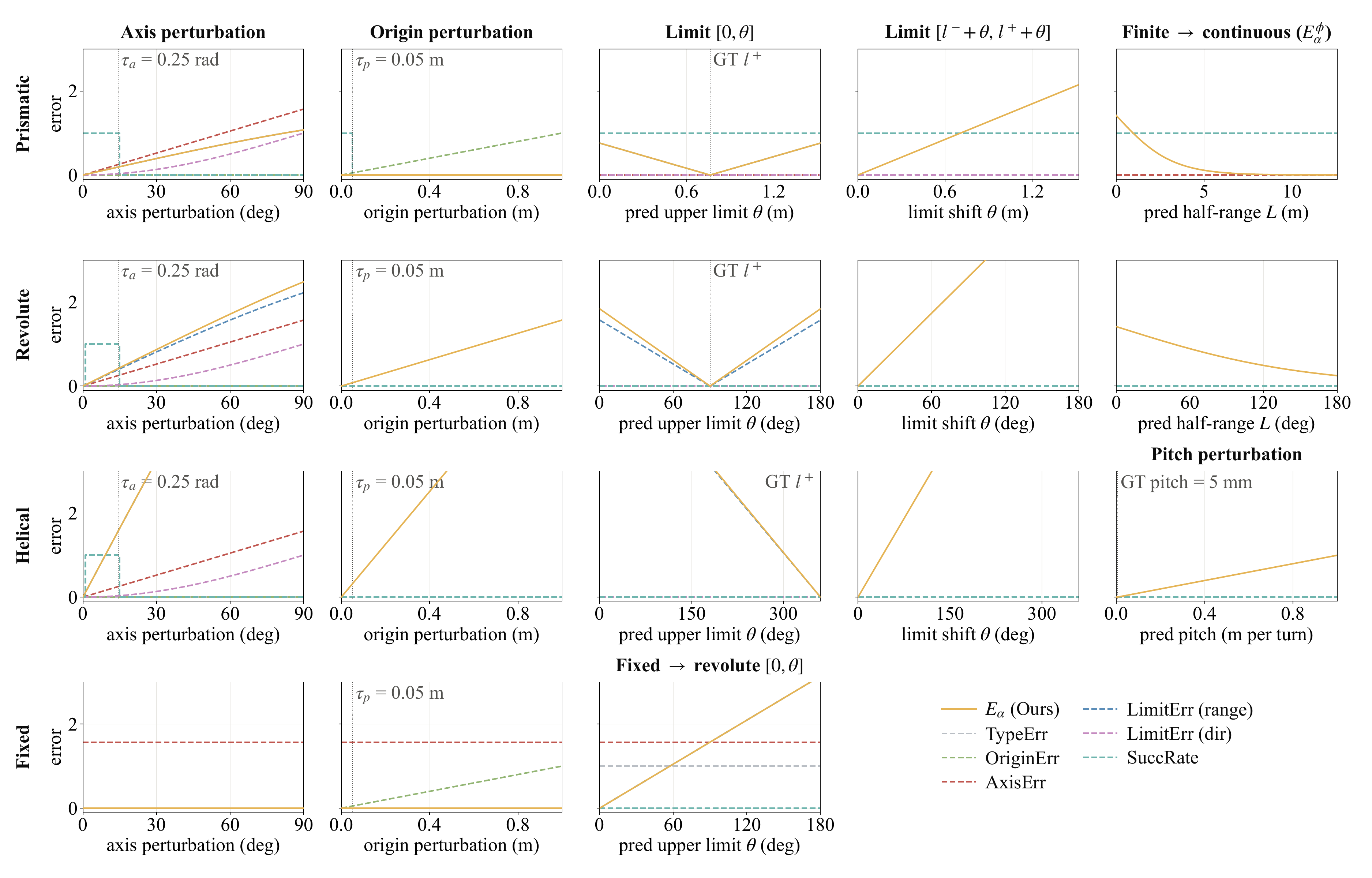}
\vspace{-20pt}
\caption{Controlled single-parameter perturbations of ground-truth library joints, one row block per joint type; the helical row composes the prismatic + continuous encoding of a screwed bulb into one screw (pitch 5\,mm per turn). Solid amber is $E_\alpha$ ($\alpha=1$; the finite $\to$ continuous column uses $E_\alpha^{\phi}$, $\kappa=\pi$), dashed curves are the component scores of \S\ref{sec:prev-eval} with the published thresholds $\tau_a = 0.25$\,rad, $\tau_p = 0.05$\,m.}
\label{fig:perturbation}
\end{figure}
Figure~\ref{fig:perturbation} sweeps one parameter of one ground-truth joint at a time and scores every step under $E_\alpha$ and under the component template. $E_\alpha$ responds smoothly and monotonically to every physical error and is exactly zero on pure gauge changes (the fixed-joint axis and origin sweeps, the prismatic origin sweep); on a prismatic axis error it follows the chord, $E_\alpha = \alpha\sqrt{(l^-)^2 + (l^+)^2}\,\cdot 2\sin(\theta/2)$. Each failure of Table~\ref{tab:failures} is visible as its own panel. The revolute origin score diverges as the axes align, so the success indicator is false even at zero error; the type error is a 0/1 cliff that $E_\alpha$ crosses continuously in the fixed $\to$ revolute column; the limit scores couple to the axis error through the motion vector $\bm{m} = \bm{a}\,(l^+ - l^-)$; and a fixed joint has a zero axis vector, for which the axis score of \S\ref{sec:prev-eval} is $0/0$, so once the denominator is regularized it evaluates to $\arccos 0 = \pi/2$ against itself and fixed joints can never be ``successful''. The pitch column has no baseline curve left to fail: the tuple taxonomy cannot express a helical joint, both sides reduce to the same revolute tuple, and every component score sits at zero for arbitrarily wrong pitch, while $E_\alpha$ grows exactly as the pitch discrepancy per turn.

\subsection{Split-norm weight}
\label{app:alpha}
\begin{figure}[!h]
\centering
\includegraphics[width=\textwidth]{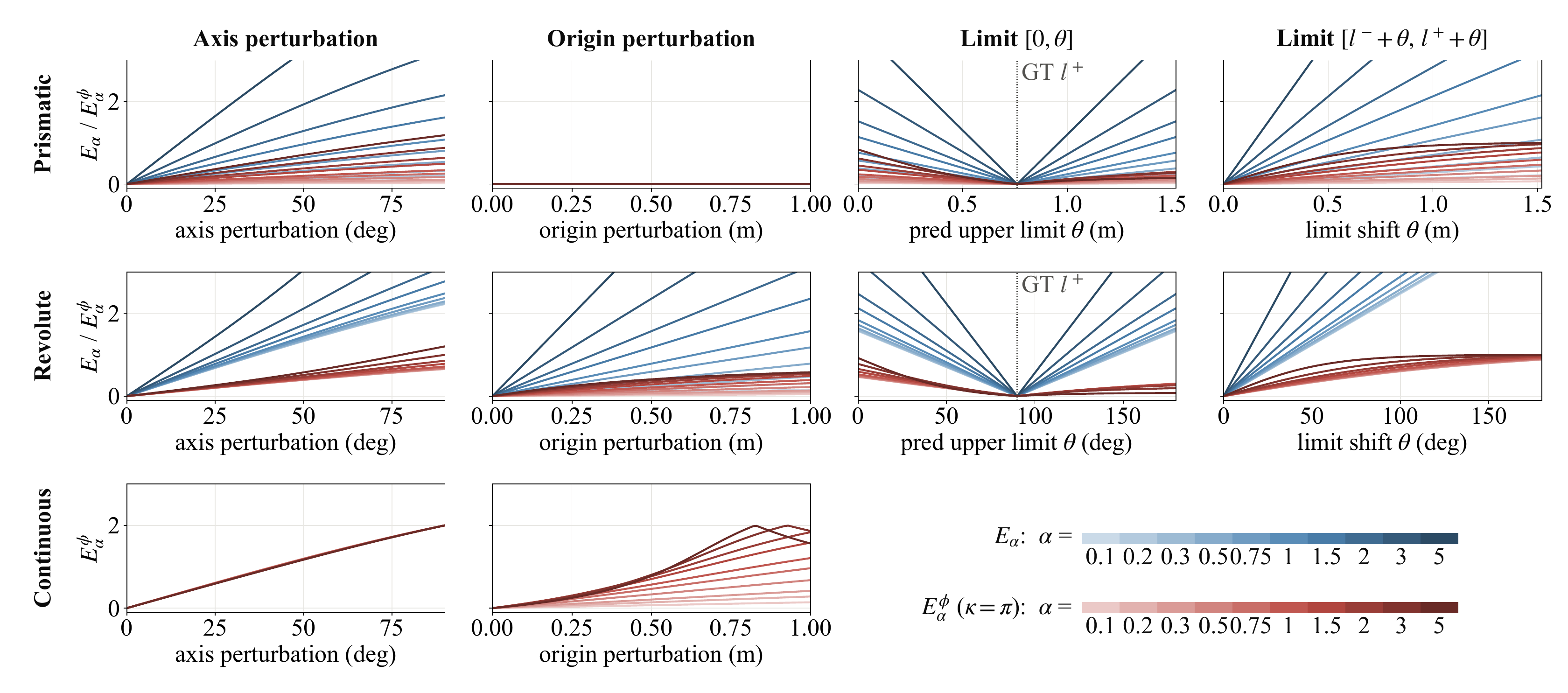}
\vspace{-20pt}
\caption{Sensitivity to the split-norm weight $\alpha \in \{0.1,\dots,5\}$, two decades around the default $\alpha=1$. Blue curves are the raw $E_\alpha$, red curves the compactified $E_\alpha^{\phi}$ ($\kappa=\pi$), lighter for smaller $\alpha$. Rows: representative prismatic, revolute, and continuous joints. Columns: axis rotation ($0..90^\circ$), origin translation ($0..1$\,m), the predicted upper limit $[0,\theta]$, and a limit shift $[l^-{+}\theta,\,l^+{+}\theta]$; the continuous row admits only the compactified score.}
\label{fig:alpha}
\end{figure}
Figure~\ref{fig:alpha} documents how the choice of $\alpha$ moves the split-norm scores. The weight rescales the translational contribution: prismatic errors scale linearly with $\alpha$, while revolute axis errors are nearly $\alpha$-independent, their endpoint difference being rotation-dominated. The prismatic origin panel is flat at zero for the whole family, so the gauge invariance of Proposition~\ref{prop:prismatic_gauge} does not depend on the choice of $\alpha$. The compactified curves are bounded and saturate, and they preserve the raw curves' ordering in $\alpha$, so the compactification does not scramble the underlying comparisons. No choice of $\alpha$ makes the split norm track physical size; that is the role of the kinetic-energy form, calibrated in Figure~\ref{fig:calibration}.

\subsection{Sweep shapes under both inner products}
\label{app:norm}
\begin{figure}[!h]
\centering
\includegraphics[width=\textwidth]{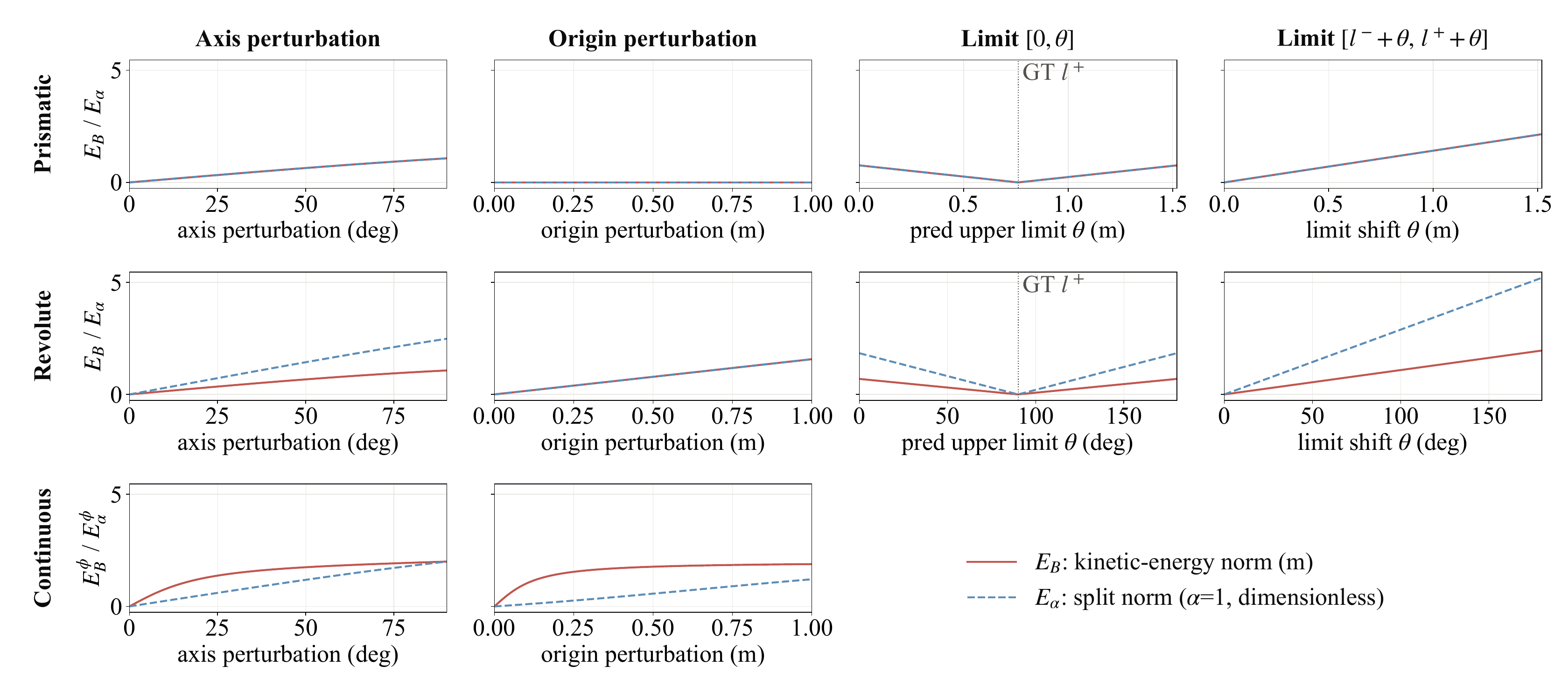}
\vspace{-20pt}
\caption{The sweeps of Figure~\ref{fig:alpha} scored under both inner products: solid brick is the kinetic-energy $E_B$ (meters), dashed blue the split-norm $E_\alpha$ ($\alpha=1$, dimensionless). The two carry different units, so only shapes may be compared, not heights. The continuous row uses the compactified score, with $\phi$ normalized by each curve's own norm.}
\label{fig:norm}
\end{figure}
Figure~\ref{fig:norm} repeats the perturbation sweeps under both inner products of \S\ref{sec:design}. Both are smooth, monotone, and gauge-invariant on every sweep, so the qualitative properties behind Table~\ref{tab:design} do not depend on the inner-product choice. On purely translational endpoint differences the two coincide exactly, since $\lVert(0,\Delta\bm{v})\rVert_B = \lVert\Delta\bm{v}\rVert = \lVert(0,\Delta\bm{v})\rVert_\alpha$ at $\alpha = 1$, which is why the prismatic panels overlap. The curves separate where rotation enters: the kinetic norm weighs a rotational error by the moving link's moments about the axis, the split norm weighs it identically for every link, and the two agree exactly on a centered link with isotropic inertia of gyration radius $1/\alpha$ (\S\ref{sec:design}), the formal sense in which a global $\alpha$ assumes every link shares one characteristic size.

\subsection{Object scale}
\label{app:scale}
\begin{figure}[!h]
\centering
\includegraphics[width=0.75\textwidth]{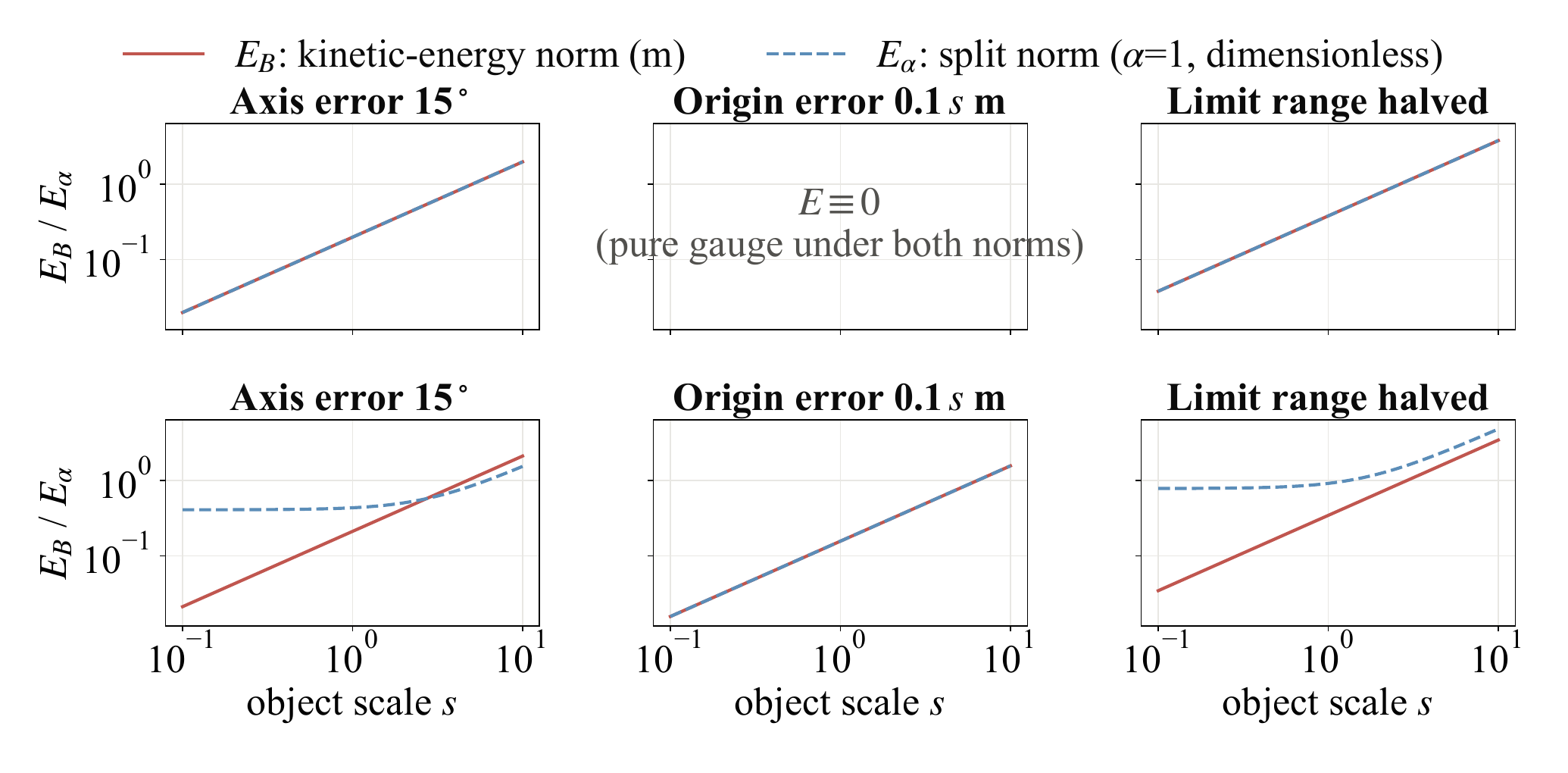}
\vspace{-20pt}
\caption{The same representation-level errors applied to the same objects uniformly scaled by $s \in [0.1, 10]$ (log--log axes): a fixed $15^\circ$ axis rotation, an origin offset of $0.1\,s$\,m, and a halved motion range. Solid brick is $E_B$, dashed blue $E_\alpha$. The prismatic origin panel carries no curves because that error is pure gauge under both norms at every scale.}
\label{fig:scale}
\end{figure}
Figure~\ref{fig:scale} isolates object size. The kinetic-energy score is exactly scale-covariant, $E_B(s) = s\,E_B(1)$, a slope-one line in every panel (uniform scaling sends $m \to s^3 m$, $\bm{c} \to s\,\bm{c}$, $I \to s^5 I$), matching its reading as RMS material displacement: the same angular error moves twice the material twice as far on a twice-as-large object. The split norm is not scale-covariant. Its rotational term is scale-free, so $E_\alpha$ flattens wherever the error has an angular component and follows $s$ only where it is purely translational. A fixed global $\alpha$ therefore under-weights errors on large objects and over-weights them on small ones relative to physical displacement, the controlled counterpart of the calibration gap in Figure~\ref{fig:calibration}.

\subsection{Compactification scale}
\label{app:kappa}
\begin{figure}[!htbp]
\centering
\includegraphics[width=\textwidth]{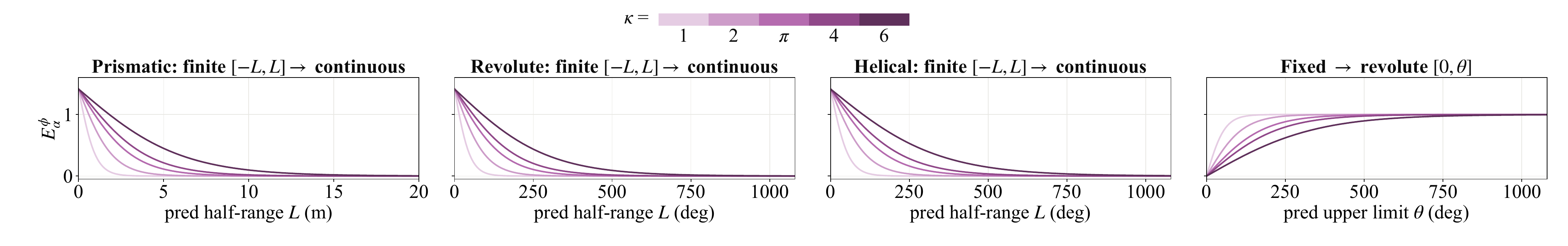}
\vspace{-20pt}
\caption{Sweep of the compactification scale $\kappa$. Curves show $E_\alpha^{\phi}$ ($\alpha=1$) for $\kappa\in\{1,2,\pi,4,6\}$, lighter for smaller $\kappa$. Left three panels: a ground-truth prismatic, revolute, or helical joint of unbounded range ($l^\pm=\pm\infty$, the continuous case of each type) against a finite prediction $[-L,L]$ as $L$ grows. Right: a welded ground-truth joint against a revolute prediction whose range $[0,\theta]$ opens from zero.}
\label{fig:kappa}
\end{figure}
Figure~\ref{fig:kappa} sweeps $\kappa$ around the default $\pi$. Against an unbounded ground truth the closed form is $E^{\phi} = \sqrt{2}\,\big(1-\tanh(L\lVert\bm{\xi}\rVert/\kappa)\big)$ (Proposition~\ref{prop:continuous_limit}). Every curve starts at $\sqrt{2}$, where the finite joint is the welded joint, and decays monotonically to zero, so an unbounded joint is literally the limit of a growing finite range, with no formula switch. The fourth panel shows the other boundary, $E^{\phi}=\tanh(\theta\lVert\bm{\xi}\rVert/\kappa)$ rising from zero, the compactified form of Proposition~\ref{prop:fixed_continuity}. Across all panels $\kappa$ acts as a pure horizontal dilation, with half decay near $L\lVert\bm{\xi}\rVert \approx 0.55\,\kappa$, and the default $\kappa=\pi$ keeps ranges up to a full turn in the responsive part of the curve.

\begin{figure}[!h]
\centering
\includegraphics[width=0.66\textwidth]{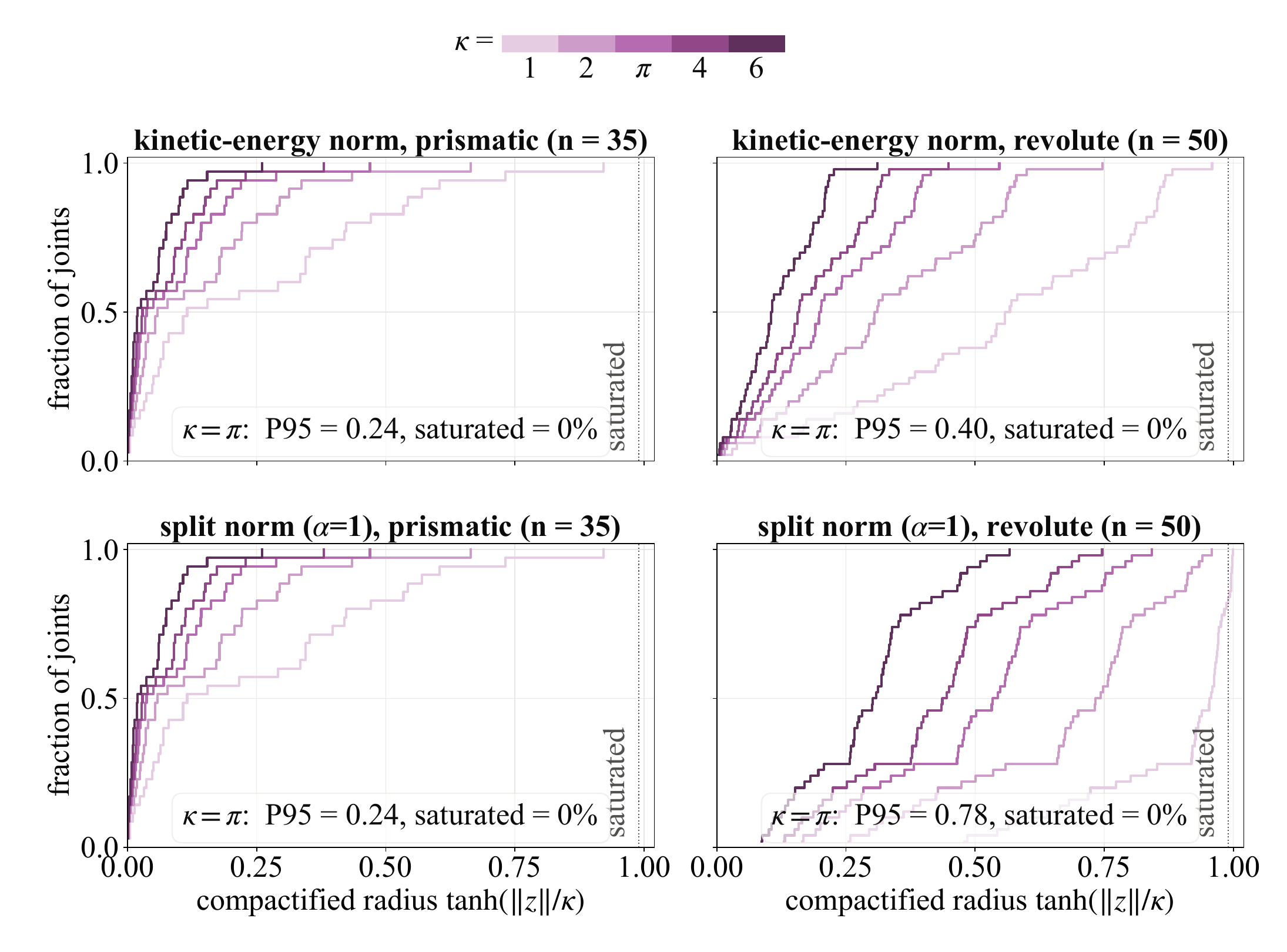}
\vspace{-15pt}
\caption{Compactified endpoint radius $\tanh(\lVert\bm{z}\rVert/\kappa)$ of the finite ground-truth joints of the evaluation set (per joint the larger of its two endpoints), as a CDF for each $\kappa \in \{1, 2, \pi, 4, 6\}$, lighter for smaller $\kappa$. Rows: the norm inside $\phi$ (kinetic-energy, split); columns: joint type. The dotted line marks radius $0.99$, where $\tanh$ saturation would begin to collapse resolution; each panel is annotated with the $\kappa=\pi$ statistics.}
\label{fig:radius}
\end{figure}

Figure~\ref{fig:radius} checks the default against the benchmark's own joints. The map $\phi$ reserves the unit sphere for continuous joints, so finite joints must stay clearly inside it, and at $\kappa = \pi$ they do: no finite ground-truth joint saturates under either norm, with 95th-percentile radii of $0.40$ (kinetic) and $0.78$ (split) on revolute joints and $0.24$ on prismatic ones, whereas $\kappa = 1$ already pushes 18\% of the split-norm revolute joints past $0.99$. The default keeps every finite joint in the responsive part of $\tanh$ without crowding the boundary reserved for continuous joints.

\subsection{Endpoint sampling}
\label{app:sampling}
\begin{figure}[!h]
\centering
\includegraphics[width=\textwidth]{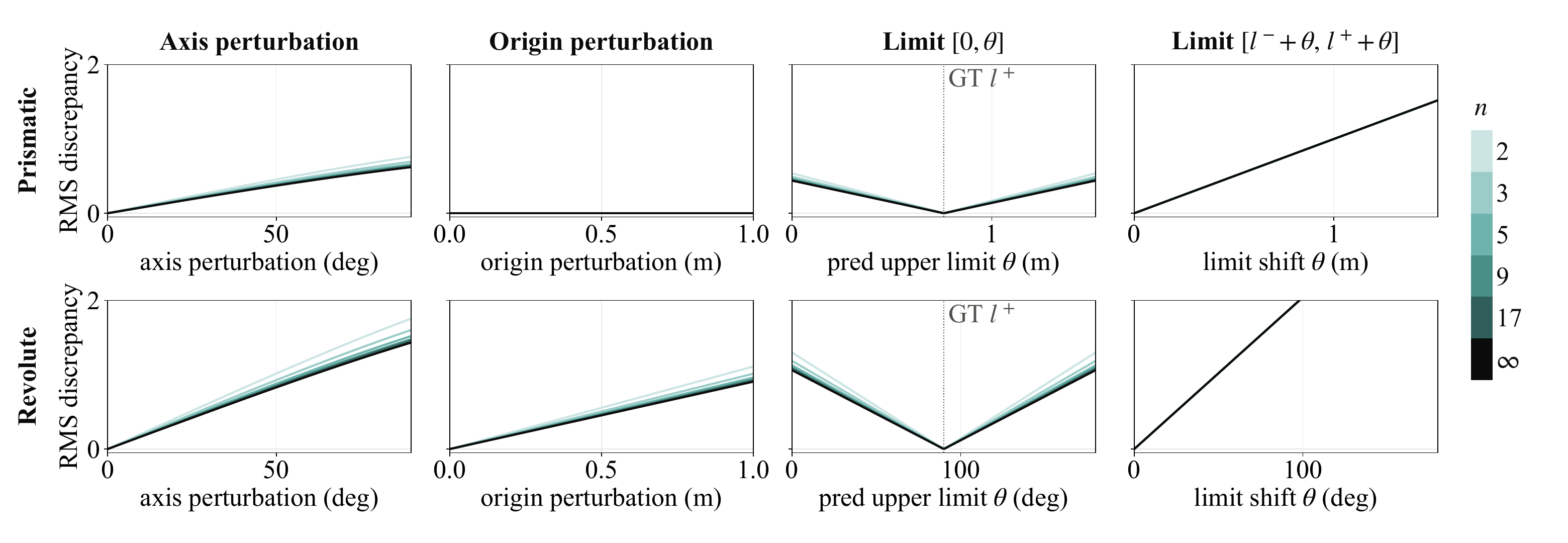}
\vspace{-20pt}
\caption{$E_n$ compares $n$ uniform samples of the linearized segments under the same $\mathbb{Z}_2$ quotient, in RMS convention so that $E_2 = E_\alpha/\sqrt{2}$ is the endpoint metric. Rows: prismatic and revolute joints. Columns: axis, origin, and limit sweeps. Darker curves use more samples; the dashed line is the closed-form $n\to\infty$ limit $D_{\mathrm{path}}$ of Theorem~\ref{thm:endpoint_path}.}
\label{fig:sampling}
\end{figure}
Figure~\ref{fig:sampling} evaluates denser sampling of the linearized motion segment against the two-endpoint representation. In every panel the curves collapse. Because $\lVert\bm{\Delta}(t)\rVert^2$ is quadratic in $t$, three points already integrate it exactly, and $n \ge 3$ is indistinguishable from the $n\to\infty$ limit $D_{\mathrm{path}}$. Pooled over all sweeps, the rank correlation with $E_2$ is at least $0.996$ for every $n$, and the ratio $D_{\mathrm{path}}/E_2$ stays inside the band $[1/\sqrt{3},\,1]$ of Theorem~\ref{thm:endpoint_path}. Denser sampling is a bounded monotone reparameterization of the endpoint metric that costs $n/2$ times more and adds no discriminative information. A continuous joint has no segment to sample, which is what the compactification handles.

\section{Additional evaluation results}
\label{app:results}
This section records the composition of the evaluation set (Appendix~\ref{app:sample}), discusses the entries of Table~\ref{tab:main} in detail (Appendix~\ref{app:detailed-comparison}), reports the evaluation counts behind them (Appendix~\ref{app:coverage}), and repeats the comparison on the objects shared by all methods (Appendix~\ref{app:intersection}).

\subsection{The evaluation set}
\label{app:sample}
The 200 objects of Table~\ref{tab:main} are chosen from \textbf{ArticulateArena-20K} independently of every method.

\textbf{Coverage.} The evaluation set spans all 15 supercategories of Table~\ref{tab:supercats} and 113 of the 628 categories, so no part of the taxonomy is left out, and within that coverage it concentrates on the supercategories that robots interact with most often. Kitchen \& Cooking, Electronics, Containers, Architectural Fixtures, and Tools contribute 28, 26, 25, 22, and 20 objects, together more than half of the set, because doors, drawers, appliances, containers, and hand tools are the articulated objects a manipulation policy meets first, while the two smallest supercategories contribute three objects each.

\textbf{Size.} The selection was made before any of the eight methods was run, and the same 200-object list is presented to all of them. Each method is scored on the objects it generates, with coverage reported in Table~\ref{tab:coverage} and controlled for on the common set in Appendix~\ref{app:intersection}. Two hundred objects support the claims Table~\ref{tab:main} makes, as the bootstrap intervals of the next paragraph show, while keeping the comparison affordable, since one row means running a published method end to end and several of the methods are generation pipelines. The library is released with verified kinematics, so larger or differently stratified evaluations can be built on it.

\textbf{What the evaluation set separates.} The claims read off Table~\ref{tab:main} are coarse, that the quotient scores induce different orderings on these outputs than the component protocol and that Articulate-Anything, fifth on $E_B$ among the methods that predict their own motion limits, has the lowest object-level $E_\alpha^{\phi,\mathrm{tree}}$. Table~\ref{tab:treeci} reports the object-level score with 95\% percentile bootstrap intervals over objects. The intervals have half-widths of $0.06$--$0.09$ on means of $0.39$--$0.59$, the interval of the leading method overlaps those of three of the other seven and is disjoint from the remaining four, and a paired bootstrap leaves the gap between the top two undecided. The order of the three leading methods is therefore left open, and the difference in orderings does not rest on any difference below the resolution of the evaluation set.

\begin{table}[htbp]
\centering
\footnotesize
\setlength{\abovecaptionskip}{0pt}
\setlength{\belowcaptionskip}{6pt}
\caption{The object-level score $E_\alpha^{\phi,\mathrm{tree}}$ with 95\% percentile bootstrap intervals over objects, on the 200-object evaluation set. The resampling unit is the object, not the joint, and each method is resampled within its own valid outputs. A paired bootstrap over the 180 objects that the two leading methods both score puts their difference at $-0.039$ with a 95\% interval of $[-0.091, 0.012]$. Articulate AnyMesh$^{\dagger}$ carries the ground-truth-imputed limits of Table~\ref{tab:main} and covers 96 objects, so its interval is wider than the others.}
\label{tab:treeci}
\setlength{\tabcolsep}{6pt}
\begin{tabular}{@{}lc@{}}
\toprule
Method & $E_\alpha^{\phi,\mathrm{tree}}$ \\
\midrule
Articraft            & 0.476\,[0.415, 0.539] \\
Articulate AnyMesh$^{\dagger}$ & 0.528\,[0.436, 0.624] \\
Articulate-Anything  & 0.392\,[0.337, 0.449] \\
ArtLLM               & 0.585\,[0.520, 0.651] \\
Ditto                & 0.552\,[0.488, 0.617] \\
Particulate          & 0.419\,[0.363, 0.476] \\
SPARK                & 0.520\,[0.459, 0.583] \\
URDFormer            & 0.566\,[0.493, 0.642] \\
\bottomrule
\end{tabular}
\end{table}

\subsection{Detailed comparison}
\label{app:detailed-comparison}
Among the methods that predict their own motion limits in Table~\ref{tab:main}, Particulate has the lowest $E_B$ (1.066\,m), while Articraft has the highest success rate (17.1\%). The two orderings differ most for Ditto, whose success rate is only 5.5\% yet whose $E_B$ (1.365\,m) is third, behind Particulate and Articraft but ahead of SPARK (1.543\,m) and Articulate-Anything (1.702\,m).

Particulate has the lowest type, axis, and limit errors and the third-highest success rate, and it also has the lowest $E_B$ among methods that predict their own limits; its origin error is larger than Articulate AnyMesh's, but for revolute joints an offset axis changes the induced motion of the ground-truth body, which $E_B$ measures together with axis direction and range, so the component view and the joint metric need not agree. The two views disagree more sharply in the other direction. Articulate-Anything has the fourth-largest type error and no component-level lead, yet it has the lowest object-level $E_\alpha^{\phi,\mathrm{tree}}$ (0.392). Under the geometry-free split norm, Particulate has the lowest joint-level $E_\alpha^{\phi}$ (0.465). Articulate AnyMesh has the lowest origin error but predicts no motion limits, so its quotient entries in Table~\ref{tab:main} impute the ground-truth range and read as geometry-only diagnostics rather than competitive scores; its imputed $E_B$ falls below Particulate's precisely because the range error is given away. These cases illustrate why component-wise scores do not induce a canonical ordering. Figure~\ref{fig:qualitative} shows three representative cases, two of which are objects on which URDFormer produced no valid output.

\begin{figure*}[htbp]
\centering
\includegraphics[width=\linewidth]{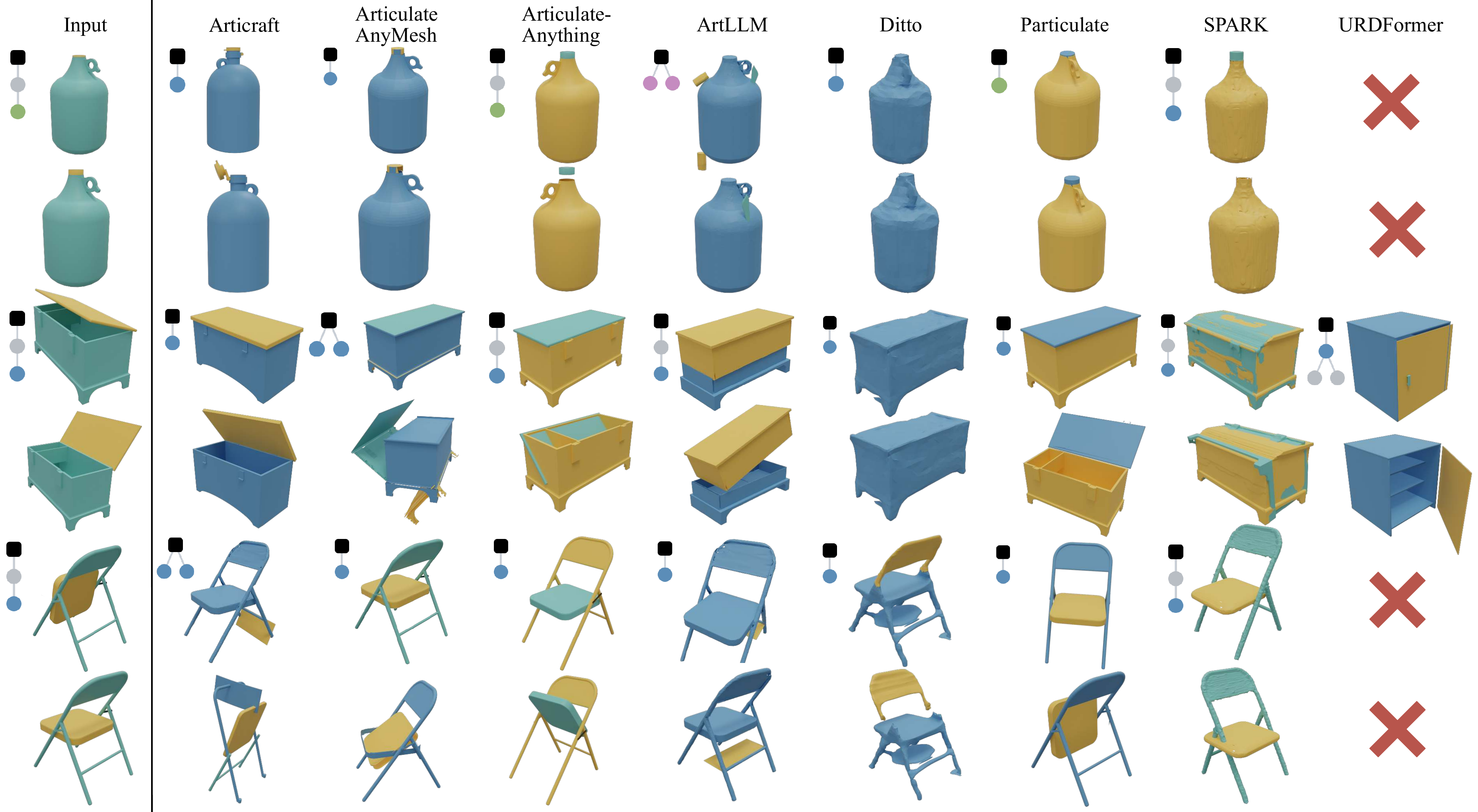}
\vspace{-20pt}
\caption{Qualitative comparison on three objects from the evaluation set of Table~\ref{tab:main}, one per row: a bottle, a box with a hinged lid, and a folding chair. The first column is the input and the remaining columns are the methods. In each panel the top render is the rest configuration and the bottom render actuates the joints, the ground-truth ones in the input column, and the kinematic tree a method returns is drawn beside its render, as in Figure~\ref{fig:type}. A red cross marks a generation failure, so URDFormer returns no valid asset for two of the three objects. The table aggregates all valid outputs, not only these three examples.}
\label{fig:qualitative}
\end{figure*}

\subsection{Evaluation coverage}
\label{app:coverage}
Table~\ref{tab:coverage} reports the evaluation counts omitted from the main table. Gen counts objects with a valid articulated output, and Pair counts the joint pairs we pair with finite predicted and ground-truth limits, which is where $E_B$ is defined; the gap between the columns is the continuous joints and the limit-free predictions that the restriction drops. For each matched ground-truth joint, $E_B$ uses the same ground-truth child-link geometry in the relevant comparison frame; predicted meshes never enter the norm. Particulate and ArtLLM use this geometry in their preprocessed frame with a geometry-only scale for canonical meters, while URDFormer uses the benchmark's camera registration.

\begin{table}[htbp]
\centering
\footnotesize
\setlength{\tabcolsep}{8pt}
\caption{Coverage of the 200-object evaluation in Table~\ref{tab:main}. Gen counts objects, and Pair counts the joint pairs with finite predicted and ground-truth limits. $^{\dagger}$Counts under the ground-truth-imputed limits of Table~\ref{tab:main}.}
\label{tab:coverage}
\begin{tabular}{@{}lrr@{}}
\toprule
Method & Gen & Pair \\
\midrule
Articraft           & 199 & 158 \\
Articulate AnyMesh$^{\dagger}$ &  96 &  79 \\
Articulate-Anything & 197 & 162 \\
ArtLLM              & 186 & 105 \\
Ditto               & 200 & 165 \\
Particulate         & 182 & 150 \\
SPARK               & 199 & 164 \\
URDFormer           & 154 & 125 \\
\bottomrule
\end{tabular}
\end{table}

\subsection{Common-set sensitivity}
\label{app:intersection}
The main table reports each method on the objects for which it produces a valid output. To check whether this coverage difference changes the quotient ranking, we repeat the comparison on the 67 objects for which all eight methods produce a valid output. Table~\ref{tab:intersection} reports means and percentile 95\% bootstrap intervals over these object IDs, with the resampling unit the object rather than the individual joint. Each score is reported where it is defined, $E_B$ on the pairs whose predicted and ground-truth limits are both finite, a count that varies by method, and the three compactified scores on all 67 objects, since they admit continuous ground-truth joints. The ordering is stable across all four scores at the ends and unresolved in the middle. Among the methods that predict their own motion limits, Articulate-Anything is in the top two under every score, joined there by Ditto under the two kinetic scores and by Particulate under the two split-norm scores. ArtLLM is last under all four, and under the three joint-level scores its interval is disjoint from every other method; between the remaining methods the spacing is below the resolution of the intervals, so we do not read a rank off it. Articulate AnyMesh$^{\dagger}$ is reported under the ground-truth-imputed limits of Table~\ref{tab:main}, so the orderings above cover only the methods that predict their own motion limits.

\begin{table}[htbp]
\centering
\scriptsize
\setlength{\tabcolsep}{3pt}
\caption{Sensitivity analysis on the 67-object intersection shared by all eight methods. Entries are means with percentile 95\% bootstrap intervals in brackets. $E_B$ is computed on the pairs with finite predicted and ground-truth limits; the other three scores are dimensionless and defined on all 67 objects. $^{\dagger}$Articulate AnyMesh predicts no motion limits, so its limits are imputed from the ground truth as in Table~\ref{tab:main} and its row reads as a geometry-only diagnostic rather than a competitive score.}
\label{tab:intersection}
\begin{tabular}{@{}lrrrr@{}}
\toprule
Method & $E_B$ (m) & $E_B^{\phi}$ & $E_\alpha^{\phi}$ & $E_\alpha^{\phi,\mathrm{tree}}$ \\
\midrule
Articraft           & 1.217\,[0.762, 1.927] & 0.564\,[0.420, 0.715] & 0.702\,[0.567, 0.846] & 0.539\,[0.423, 0.661] \\
Articulate AnyMesh$^{\dagger}$ & 1.039\,[0.548, 1.836] & 0.510\,[0.360, 0.671] & 0.590\,[0.450, 0.739] & 0.556\,[0.436, 0.682] \\
Articulate-Anything & 1.096\,[0.549, 2.044] & 0.412\,[0.309, 0.527] & 0.458\,[0.358, 0.567] & 0.303\,[0.223, 0.389] \\
ArtLLM              & 3.249\,[2.885, 3.616] & 1.046\,[0.954, 1.136] & 1.183\,[1.074, 1.283] & 0.582\,[0.468, 0.697] \\
Ditto               & 0.932\,[0.512, 1.644] & 0.460\,[0.338, 0.590] & 0.602\,[0.496, 0.717] & 0.544\,[0.428, 0.665] \\
Particulate         & 1.176\,[0.870, 1.500] & 0.515\,[0.409, 0.629] & 0.465\,[0.373, 0.556] & 0.422\,[0.335, 0.511] \\
SPARK               & 1.121\,[0.592, 2.064] & 0.461\,[0.349, 0.584] & 0.691\,[0.589, 0.798] & 0.498\,[0.393, 0.607] \\
URDFormer           & 1.409\,[0.983, 2.113] & 0.567\,[0.458, 0.685] & 0.730\,[0.636, 0.826] & 0.557\,[0.442, 0.677] \\
\bottomrule
\end{tabular}
\end{table}

\section{The ArticulateArena-20K library}
\label{app:library}
This section documents how the library is assembled, while \S\ref{sec:dataset} summarizes its composition and statistics.

\subsection{Construction}
\label{app:construction}
\textbf{Collected assets.} The library merges Articraft-10K \citep{zhou2026articraft}, Artiverse \citep{iliash2026artiverse}, Articulated-Object-Code \citep{gao2026lam}, PartNet-Mobility \citep{xiang2020sapien}, GAPartNet-AKB48, and a tail of smaller sources with a collection we generated ourselves. Figure~\ref{fig:dataset-overview}(c) breaks each supercategory down by source. The public sources arrive with different directory layouts, mesh formats, and metadata schemas. Each asset is re-packaged into a common layout with one URDF per object and its meshes resolved, and every object receives a preview rendered with a shared camera and lighting, so that downstream tooling touches a single format.

\textbf{Generated assets.} The generated collection consists of objects we produced with several generative models and, in part, by hand. Each object is authored directly as parts connected by joints, so its kinematics is specified at creation rather than annotated afterwards, and it enters the library through the same normalization and verification as every collected asset.

\textbf{Categorization.} Every object carries a category label, read from per-object metadata where the source provides one and derived from the descriptive object name for Articraft-10K by truncating at connectives and dropping generic modifiers. Labels are normalized to a shared lowercase form so that variant spellings of one class merge across sources, yielding the 628 categories, which we group into the 15 supercategories of Table~\ref{tab:supercats} (Appendix~\ref{app:supercats}). An LLM assisted this step by proposing a label for every object from its name, description, and renders, and the proposals were reviewed and corrected by the authors, so every final label passed a human check.

\textbf{Quality control.} We audit the library in a browser-based viewer that renders each object, drives every movable joint through its limits, and records per-object issue reports of missing textures, broken meshes or joints, and wrong orientation or scale. We use these reports to repair or exclude assets.

\subsection{Supercategory taxonomy}
\label{app:supercats}
The categories of \textbf{ArticulateArena-20K} are grouped into 15 supercategories. Table~\ref{tab:supercats} summarizes the mapping, and Figure~\ref{fig:longtail} shows the long-tailed category distribution. The full category list ships with the library.

\makeatletter
\setlength{\@fptop}{0pt}\setlength{\@fpsep}{12pt}\setlength{\@fpbot}{0pt plus 1fil}
\makeatother
\begin{figure}[!htbp]
\centering
\includegraphics[width=0.70\textwidth]{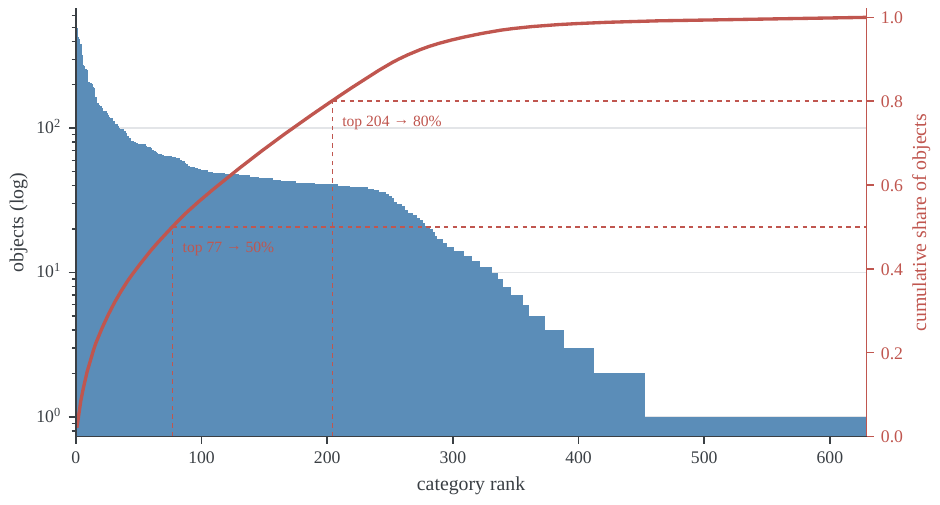}
\setlength{\abovecaptionskip}{4pt}
\caption{Category sizes of \textbf{ArticulateArena-20K} by rank (bars, log scale) and the cumulative share of objects (curve). The largest 77 categories hold half of the objects and the largest 204 hold 80\%.}
\label{fig:longtail}
\end{figure}

\begin{table}[!htb]
\centering
\footnotesize
\renewcommand{\arraystretch}{0.95}
\setlength{\abovecaptionskip}{4pt}
\setlength{\belowcaptionskip}{4pt}
\setlength{\tabcolsep}{4pt}
\caption{Supercategories of \textbf{ArticulateArena-20K}, ordered by size, with the three largest categories of each and their object counts.}
\label{tab:supercats}
\begin{tabular}{@{}lrr>{\raggedright\arraybackslash}p{0.5\textwidth}@{}}
\toprule
Supercategory & Objects & Categories & Largest categories \\
\midrule
Storage Furniture & 2{,}689 & 31 & nightstand (490), chest\_of\_drawers (430), storage\_furniture (381) \\
Mechanisms & 2{,}324 & 80 & branching\_tree (77), prismatic\_slider (63), yaw\_axis\_module (62) \\
Kitchen \& Cooking & 1{,}983 & 48 & refrigerator (270), oven (210), dishwasher (190) \\
Tools & 1{,}890 & 125 & scissors (192), pliers (106), stapler (103) \\
Electronics & 1{,}789 & 43 & laptop (415), remote\_control (138), switch (128) \\
Architectural Fixtures & 1{,}645 & 48 & faucet (254), door (207), window (150) \\
Mobility & 1{,}262 & 70 & cart (82), kick\_scooter (79), car (54) \\
Tables \& Seating & 1{,}167 & 33 & desk (321), table (132), chair (131) \\
Containers & 1{,}080 & 22 & bottle (163), bucket (143), trash\_can (117) \\
Lighting & 806 & 20 & floor\_lamp (112), task\_lamp (82), ceiling\_light\_fixture (78) \\
Infrastructure & 805 & 26 & ferris\_wheel (80), clock\_tower (51), elevator (50) \\
Home Appliances & 798 & 20 & fan (98), washing\_machine (88), ceiling\_fan (77) \\
Personal \& Household Items & 756 & 22 & eyeglasses (272), globe (118), clock (78) \\
Robotics & 517 & 15 & robotic\_arm (92), opposed\_twinslide\_gripper (47), robotic\_leg (47) \\
Recreation & 466 & 25 & fidget\_toy (54), dj\_equipment (53), playground\_swing (42) \\
\midrule
Total & 19{,}977 & 628 & \\
\bottomrule
\end{tabular}
\end{table}

\end{document}